\documentclass[preprint,12pt]{elsarticle}

\usepackage{amssymb}
\usepackage{mathrsfs}
\usepackage{mathrsfs}
\usepackage{mathrsfs}
\usepackage{amsfonts}
\usepackage{amsfonts}
\usepackage{amsfonts}
\usepackage{mathrsfs}
\usepackage{amsmath}
\usepackage{graphicx}
\usepackage{multirow}
\usepackage{caption2}
\usepackage{float}
\usepackage{booktabs}
\usepackage{algorithm}
\usepackage{algorithmic}
\usepackage{subfigure}
\newtheorem{theorem}{Theorem}[section]

\numberwithin{equation}{section}

\newtheorem{lemma}[theorem]{Lemma}

\newtheorem{remark}[theorem]{Remark}

\newenvironment{proof}[1][Proof]{\textbf{#1. }}{\ \rule{0.5em}{0.5em}}%
\journal{}

\begin{document}

\begin{frontmatter}



\title{Greedy metrics in orthogonal greedy learning \tnoteref{t1}} \tnotetext[t1]{The research was
supported by the National 973 Programming (2013CB329404), the Key
Program of National Natural Science Foundation of China (Grant No.
11131006) and the  National Natural Science Foundation of China
(Grant No. 11401462)}

\author{Lin Xu$^1$  }

\author{Shaobo Lin$^{2}$\corref{*}}\cortext[*]{Corresponding author: sblin1983@gmail.com}

\author{Jinshan Zeng$^1$}

\author{Zongben Xu$^1$}

\address{1. Institute for Information and
System Sciences, School of Mathematics and Statistics, Xi'an
Jiaotong University, Xi'an, 710049, China

2. College of Mathematics and Information Science, Wenzhou
University, Wenzhou 325035, China }

\begin{abstract}

Orthogonal greedy learning (OGL) is a stepwise learning scheme that
adds a new atom from a dictionary via the steepest gradient descent
and build the estimator via orthogonal projecting the target
function to the space spanned by the selected atoms in each greedy
step. Here, ``greed'' means choosing a new atom according to the
steepest gradient descent principle. OGL then avoids the
overfitting/underfitting by selecting an appropriate iteration
number. In this paper, we point out that the
overfitting/underfitting can also be avoided via redefining
``greed'' in OGL. To this end, we introduce a new greedy metric,
called $\delta$-greedy thresholds, to refine ``greed'' and
  theoretically verifies its feasibility. Furthermore, we
 reveals that such a
greedy metric can bring an adaptive termination rule on the premise
of maintaining the prominent learning performance of OGL.   Our
results show that the steepest gradient descent is not the unique
greedy metric of OGL and some other more suitable metric may lessen
the hassle of model-selection of OGL.

\end{abstract}

\end{frontmatter}
% Introduction
%%%%%%%%%%%%%%%%%%%%%%%%%%%%%%%%%%%%%%%%%%%%%%%%%%%%
\section{Introduction}

Supervised learning focuses on  synthesizing a function (or mapping)
to approximate (or represent) an
%{\color{red} unknown but definite}
underlying
 relationship
between the input and corresponding output based on finitely many
input-output samples.  A system tackling  supervised learning
problems is commonly called as a learning system (or learning
machine). A standard learning system usually comprises a hypothesis
space, an optimization strategy, and a learning algorithm;
Specifically, the hypothesis space is a family of parameterized
functions that encodes the prior knowledge of the data, and the
optimization strategy is an optimization problem which defines the
estimator by utilizing the given samples, and the learning algorithm
is an inference procedure that numerically solves the optimization
problem.
%an optimization strategy,  an optimization
%problem which defines the estimator by utilizing the given samples;
%and a learning algorithm, a numerical inference procedure that
%numerically solves the optimization problem.

Dictionary learning  is a family of learning systems   whose
hypothesis spaces are  linear combinations  of atoms (or elements)
of some given dictionaries. Here,  the dictionary denotes a family
of base learners \cite{Temlaykov2008}. For such type hypothesis
spaces,  regularization schemes such as the bridge estimator
\cite{Armagan2009}, ridge estimator \cite{Golub1979} and Lasso
estimator \cite{Tibshirani1995} are often employed as the
optimization strategies. When the scale of samples is not too large,
these optimization strategies can be realized by various learning
algorithms such as the regularized least square algorithms
\cite{Wu2006}, iterative thresholding algorithms
\cite{Daubechies2004} and iterative reweighted algorithms
\cite{Daubechies2010}. However, a large portion of the
aforementioned learning algorithms are time-consuming and therefore
may cause the sluggishness of the corresponding learning systems
\cite{Zhang2014}, particularly, when
%faced with
applied to the large-scale data sets.

Greedy learning  or, more specifically,  learning through greedy
search or applying greedy-type algorithms, provides a possibility to
circumvent the  drawbacks  of regularization methods
\cite{Barron2008}.  Greedy-type algorithms are stepwise inference
processes that start  from a null model and follow  the problem
solving heuristic of making the locally optimal choice at each step
with the hope of finding a global optimum. If the number of steps is
moderate, then greedy-type algorithms possess charming computational
advantage, when compared with the regularization schemes
\cite{Temlaykov2008}.
 This
property triggers avid research activities of greedy-type algorithms
in signal processing \cite{Dai2009,Kunis2008,Tropp2004}, inverse
problem \cite{Donoho2012,Tropp2010}, sparse approximation
\cite{Donoho2007,Temlaykov2011} and,
%of course,
particularly,
machine learning
\cite{Barron2008,Chen2013a,Lin2013a}.

\subsection{Elements of greedy learning}

Four most important  elements  of greedy learning are the
``dictionary-selection'', ``greedy-metric'', ``iterative-strategy''
and ``stopping-criterion''. This is essentially different from the
greedy approximation that usually only focuses on the
``dictionary-selection'' and ``iterative-format'' issues
\cite{Temlaykov2008}, as the greedy learning concerns not only the
approximation capability, but also the  cost, such as the model
complexity,  that should pay to achieve a specified approximation
accuracy. Therefore, greedy learning can be regarded as a four-issue
learning scheme.

$\bullet$  ``Dictionary-selection'' issue: this issue devotes to
selecting a suitable dictionary for a given learning task. As a
classical
%long-standing
topic of greedy approximation, there are a
great deal of  dictionaries available to greedy learning. Typical
examples include   the greedy basis \cite{Temlaykov2008},
quasi-greedy basis \cite{Temlaykov2003}, redundant dictionary
\cite{Devore1996}, orthogonal basis \cite{Temlyakov1998},
kernel-based sample dependent dictionary \cite{Chen2013,Lin2013a}
and stump dictionary  \cite{Friedman2001}.

$\bullet$ ``Greedy-metric'' issue: this issue regulates the
criterion to choose a new atom (or element) from the dictionary  in
each greedy step.
Besides the widely used steepest gradient
descent (SGD) method
%metric
\cite{Devore1996}, there are also many existing methods such as weak
greed \cite{Temlaykov2000}, thresholding greed \cite{Temlaykov2008}
and super greed \cite{Liu2012} to quantify the greedy-metric for the
approximation purpose.  However, to the best of our knowledge, only
the SGD metric is employed in greedy learning, as all the results in
\cite{Liu2012,Temlaykov2000,Temlaykov2008} imply that this metric is
superior to other metrics  in greedy approximation.

$\bullet$ ``Iterative-format'' issue: this issue focuses on how to
define a new
 estimator based on the selected atoms.
Similar to the ``dictionary-selection'' issue, the
``iterative-strategy'' issue is also a classical
%long-standing
 topic of greedy approximation.
There are several existing types of greedy iteration schemes
\cite{Temlaykov2008}. Among these, three most commonly used
iteration schemes are the pure greedy, orthogonal greedy and relaxed
greedy formats. Each format possess its own pros and cons
\cite{Temlaykov2003,Temlaykov2008} and has been widely used in
greedy approximation and learning
\cite{Barron2008,Chen2013,Friedman2001,Lin2014,Temlaykov2008a}.
 For
instance, compared with the orthogonal greedy strategy, the   pure
and relaxed greedy strategies have benefits of computation but
suffer from either the low convergence rate or the small applicable
scope problem.

$\bullet$ ``Stopping-criterion'' issue: this issue depicts how to
terminate the learning process. The ``stopping-criterion'' is
regarded as the main distinction between greedy approximation and
learning and has been frequently studied recently
\cite{Barron2008,Chen2013,Lin2013a}. For example, Barron et al.
\cite{Barron2008} proposed an $l^0$-based complexity regularization
strategy, and Chen et al. \cite{Chen2013} provided an $l^1$-based
adaptive stopping criterion.

\subsection{Motivations of greedy metrics}

Orthogonal greedy learning (OGL) is a stepwise learning scheme that
adds a new atom from a dictionary via SGD and then generate an
estimator via orthogonally projecting the objective function to the
space spanned by the selected atoms at each greedy step. A common
consensus of orthogonal  greedy approximation is that  better
approximation results  can be achieved with larger number of
iterations \cite{Temlaykov2008}. However, this claim can not be
applicable to greedy learning since the estimator is based on the
samples with observational noises. Therefore, researches usually
adopt a suitable number of iteration in OGL
  to avoid the overfitting/underfitting \cite{Barron2008,Chen2013}.
\begin{figure}[H]
\centering
\includegraphics[height=6cm,width=7.0cm]{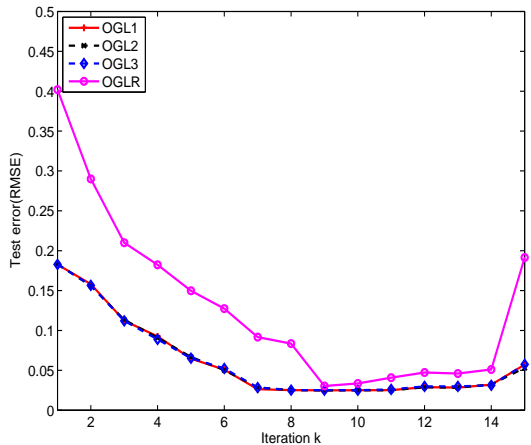}
\caption{The comparisons among four OGL with different greedy
metrics. The  levels of greed satisfies OGL1$\geq$ OGL2$\geq$ OGL3$\geq$
OGLR}
\end{figure}

Since OGL always searches the most correlative atom
  and realizes the optimal approximation capability of the space spanned by the selected
atoms in each greedy step, its generalization capability becomes
  sensitive to the number of iterations.
Thus, a slight turbulence of the number of atoms may lead to a great
change of the generalization capability, which can be witnessed in
Fig.1. Furthermore, the   $l^0$-based complexity regularization
strategy \cite{Barron2008} is only for the benefit of theoretical
analysis and the applicable range of the $l^1$-based adaptive
stopping criterion \cite{Chen2013} is quite restricted, which makes
it be difficult to persuade the programmers to utilize OGL.
Recalling that a possible reason of this problem is OGL searches the
new atom according to SGD,
 an advisable idea is to weaken the  level of
greed by taking the ``greedy-metric'' issue into account. For this
purpose, we run a simple simulation (whose experimental setting can
be found in Sec. 5.2) to judge the possibility of this idea.
%thinking.
The result (Fig.1) shows that
the generalization of OGL will not degrade via weakening the level of greed
if the greedy-metric is specified appropriately.
%, then weakening the level
%of greed does not degrade the generalization of OGL.

\subsection{Our contributions}

Different from other three issues of greedy learning, the ``greedy
metric'' issue, to the best of our knowledge, has been studied a few
in both theory and practice. The purpose of the present paper is to
reveal the importance and necessity of studying the
``greedy-metric'' issue in OGL.  The main contributions can be
summarized as the following.

$\bullet$ We propose a new greedy metric called the
``$\delta$-greedy thresholds'' to measure the level of greed  in
OGL. Although  this metric has already been used in greedy
approximation \cite{Temlaykov2008}, the novelty of translating  it
to OGL is that using this metric in OGL provides a possibility to
improve the generalization capability of OGL further. We prove that,
if the iteration number is appropriately specified, then OGL with
the ``$\delta$-greedy thresholds'' metric can reach the existing
almost optimal learning rate  of OGL \cite{Barron2008}.

$\bullet$ Based on the ``$\delta$-greedy thresholds'', an adaptive
termination rule is developed for OGL. Different from the classical
stopping criterion that reach the bias and variance balance via
choosing appropriate number of iterations, our study implies that
the balance can also be attained through setting a suitable greedy
metric. This phenomenon reveals the essential importance of the
``greedy-metric'' issue, which often seems to be   overlooked in
greedy learning. We also presents the theoretical justification of
such an adaptive termination rule. Our result  (Theorem 3.2) shows
that
 the greedy-metric based
termination  rule performs as good as the iteration number based
termination rule \cite{Barron2008} in the sense that the
generalization capabilities of the corresponding OGL are almost
identical.

%$\bullet$  Utilizing the ``$\delta$-greedy thresholds'' metric and
%its corresponding adaptive termination rule, we design a new
%learning system, called $\delta$-thresholding orthogonal greedy
%learning ($\delta$-TOGL), to tackle regression problem. The
%theoretical result (Theorem 3.2) reveals that $\delta$-TOGL can
%attain the almost optimal generalization error bound of OGL.
%Different from  OGL \cite{Barron2008,Chen2013}, the proposed
%$\delta$-TOGL system is regarded as more user-friendly. In
%particular, we run a series
% simulations to demonstrate that, besides maintaining the good
%performance of OGL, $\delta$-TOGL provides a more appropriate
%approach to tackle the
% model-selection problem. That is, our simulations show
%that CV is feasible to choose an appropriate $\delta$ of
%$\delta$-TOGL.

\subsection{Organization}
The rest of paper is organized as follows. In the next section,
  we make a brief introduction of statistical learning theory
and greedy learning. In Section 3, we introduce the
``$\delta$-greedy thresholds'' metric in OGL and provide its
feasibility justification. In Section 4, based on the
``$\delta$-greedy thresholds'' metric, we propose an adaptive
termination rule and the corresponding $\delta$-TOGL system. The
theoretical feasibility of the $\delta$-TOGL system is also given in
this section. In Section 5, we present numerical simulation
experiments to verify our arguments. In Section 6,
 we provide the   proofs
of the main results. In the last section, we draw a simple
conclusions of this paper.

\section{Preliminaries}
In this section, we present some preliminaries
% of our stations.
A fast review of the statistical learning theory as well as greedy learning is
given in Sec.2.1 and Sec.2.2, respectively.

\subsection{Statistical learning theory}
Suppose that $\mathbf{z}=(x_{i},y_{i})_{i=1}^{m}$ are drawn
independently
and identically
from $Z:=X\times Y$ according to an
unknown probability distribution $\rho $ which admits the
decomposition
$$
                 \rho (x,y)=\rho _{X}(x)\rho (y|x).
$$
Assume that $f:X\rightarrow Y$
%is a function that
characterizes the
correspondence between the input and output, as induced by $\rho $.
A natural measure of the error incurred by using $f$ of this
purpose is the generalization error, defined by
$$
             \mathcal{E}(f):=\int_{Z}(f(x)-y)^{2}d\rho ,
$$
which is minimized by the regression function \cite{Cucker2001}
$$
                f_{\rho }(x):=\int_{Y}yd\rho (y|x).
$$
%We do not know this ideal minimizer
In general, since $\rho $ is unknown, $f_{\rho }$ is also unknown.
However, we have access to random examples ${\bf z}$ from
$X\times Y$ sampled according to $\rho $.

Let $L_{\rho _{_{X}}}^{2}$ be the Hilbert space of $\rho _{X}$
square integrable functions on $X$, with norm $\Vert \cdot \Vert
_{\rho }.$ It is known that, for every $f\in L_{\rho _{X}}^{2}$,
there holds
\begin{equation}\label{equality}
          \mathcal{E}(f)-\mathcal{E}(f_{\rho })=\Vert f-f_{\rho }\Vert _{\rho
           }^{2}.
\end{equation}%
So, the goal of learning is to find a best approximation of the
regression function $f_{\rho }$.

Let $\mathcal{H}$ be a hypothesis space and $f_{\mathcal{H}}\in
\mathcal{H}$ be a best approximation of $f_{\rho },$ i.e.,
$f_{\mathcal{H}}=\arg \min_{g\in \mathcal{H}}\Vert g-f_{\rho }\Vert
_{\rho }^{2}.$ Whenever there is an estimator $f_{\bf z}\in
\mathcal{H}$ based on the samples ${\bf z}$ in hand, we
  have%
\begin{equation}\label{vias and variance}
            \mathcal{E}(f_{\bf z})-\mathcal{E}(f_{\rho })=\Vert f_{\rho }-f_{\mathcal{H}%
              }\Vert _{\rho }^{2}+\mathcal{E}(f_{\mathcal{H}})-\mathcal{E}(f_{\bf
              z}).
\end{equation}%
It is known \cite{Cucker2007} that a small $\mathcal{H}$ will derive a large bias $%
\Vert f_{\rho }-f_{\mathcal{H}}\Vert _{\rho }^{2},$ while a large $\mathcal{H%
}$ deduces a large variance
$\mathcal{E}(f_{\mathcal{H}})-\mathcal{E}(f_{\bf z}).$ Thus the bias
and  variance are conflicting, and an ideal or best hypothesis space
$\mathcal{H}^{\ast }$ should be  the one that best compromises the
bias and the variance. This is the well known "bias-variance"
dilemma in statistical learning theory.

Without loss of generality, we always assume $y \in [-M,M]$, and the
number of samples is finite.
Thus, it is reasonable to truncate the
estimator to $[-M,M]$. That is, if we define
$$
          \pi_Mu=\left\{\begin{array}{l l}
             u,       & \mbox{if}\ |u|\leq M  \\
               Msign(u),& \mbox{otherwise}
\end{array}
\right.
$$
as the truncation operator, then it is easy to deduce
\cite{Zhou2006}
 $$
    \|\pi_Mf_{\bf z}-f_\rho\|^2_\rho\leq \|f_{\bf z}-f_\rho\|^2_\rho.
 $$

\subsection{Greedy learning}

Let $H$ be a Hilbert space endowed with norm $\|\cdot\|_H$ and inner product
$\langle\cdot,\cdot,\rangle_H$.
Let
$\mathcal D=\{g\}_{g\in\mathcal D}$ be a given dictionary satisfying
$\|g\|_H\leq 1$. Define $\mathcal L_1=\{f:f=\sum_{g\in D}a_gg\}$ as
a Banach space endowed with the  norm
$$
       \|f\|_{\mathcal
           L_1}:=\inf_{\{a_g\}_{g\in\mathcal D}}\left\{\sum_{g\in \mathcal D}|a_g|:f=\sum_{g\in \mathcal
              D}a_gg\right\}.
$$

There exist several types of  greedy algorithms
\cite{Temlaykov2003}. Three most commonly used are the pure greedy (PGA),
orthogonal greedy (OGA) and  relaxed greedy (RGA) algorithms.
%which are often
%denoted by their acronyms PGA, OGA and RGA, respectively.
In all the
above greedy algorithms, we begin by setting $f_0:=0$. The new
approximation $f_k\; (k \ge 1)$ is defined based on
$r_{k-1}:=f-f_{k-1}$. In OGA, $f_k$ is defined as
$$
                f_k=P_{V_k} f,
$$
where $P_{V_k}$ is the orthogonal projection onto
$V_k=\mbox{span}\{g_1,\dots,g_k\}$ and $g_k$ is defined as
$$
        g_k=\arg\max_{g\in\mathcal D}|\langle r_{k-1},g\rangle_H|.
$$

Given a set of training samples ${\bf z}=(x_i,y_i)_{i=1}^m$, the
empirical inner product and norm are defined by
$$
        \langle f,g\rangle_m:=\frac1m\sum_{i=1}^mf(x_i)g(x_i), \
        \|f\|_m^2:=\frac1m\sum_{i=1}^m|f(x_i)|^2.
$$
 The initial setting of OGL is the same as that of OGA. However,
  OGL should take  the following four issues
into account:

{\bf (I) Dictionary-selection:} Select a dictionary $\mathcal
D_n:=\{g_1,\dots,g_n\}$ with $\|g_i\|_m\leq 1$.

{\bf (II)  Greedy-definition:}
$$
                  g_k=\arg\max_{g\in\mathcal D_n}|\langle r_{k-1},g\rangle_m|.
$$

{\bf  (III) Iteration-strategy:}
$$
                f_{\bf z }^k=P_{V_{{\bf z},k}} f,
$$
where $P_{V_{{\bf z},k}}$ is the orthogonal projection onto
$V_k=\mbox{span}\{g_1,\dots,g_k\}$ in the metric of $\|\cdot\|_m$.

 {\bf (IV)  Stopping criterion:} Terminate the learning process
when $k$ satisfies a certain assumption.

\section{Greedy-metric in OGL}

Given a real functional $V:\mathcal H\rightarrow\mathbf R$, the
Fr\'{e}chet derivative of $V$ at $f$, $V'_f:\mathcal
H\rightarrow\mathbf R$, is the linear functional such that for
$g\in\mathcal H$,
$$
  \lim_{\|g\|_{\mathcal
  H}\rightarrow0}\frac{|V(f+g)-V(f)-V'_f(g)|}{\|g\|_{\mathcal H}}=0,
$$
and the gradient of $V$ as a map $\mbox{grad}V:\mathcal
H\rightarrow\mathcal H$ is defined by
$$
           \langle \mbox{grad}V(f),g\rangle_\mathcal H=V'_f(g),\ \mbox{for
           all}\ g\in\mathcal H.
$$
The greedy-metric adopted in (II) is to find $g_k\in \mathcal D_n$
such that
$$
              \langle
              -\mbox{grad}(A_m)(f_{\bf z}^{k-1}),g_k\rangle=\sup_{g\in \mathcal D_n}\langle
              -\mbox{grad}(A_m)(f_{\bf z}^{k-1}),g\rangle,
$$
where $A_m(f)=\sum_{i=1}^m|f(x_i)^2-y_i|^2$. Therefore, the
classical greedy-metric is based on the steepest  gradient descent
of $r_{k-1}$ with respect to the dictionary $\mathcal D_n$. By
normalizing the residual $r_k$, $k=0,1,2,\dots,n$, (II) equals to
search $g_k$ satisfying
$$
              g_k=\arg\max_{g\in\mathcal D_n}\frac{|\langle r_{k-1},g\rangle_m|}{\|r_{k-1}\|_m}.
$$
Geometrically, it means to search a $g_k$ minimizing the angle
$\theta_k$ between $r_{k-1}/\|r_{k-1}\|_m$ and $g_k$, which is
depicted as the following Fig.2.
\begin{figure}[H]
\centering
\subfigure{\includegraphics[height=5cm,width=10.0cm]{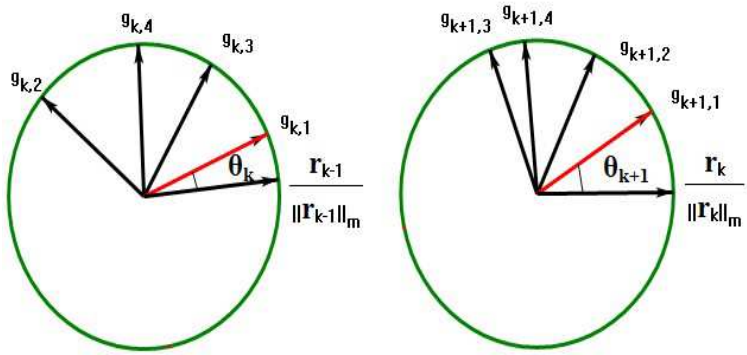}}
%\subfigure[]{\includegraphics[height=5cm,width=5.0cm]{figuress/greed2.eps}}
\caption{Classical greedy-metric}
\end{figure}

Recalling the definition of OGL, it is not difficult to judge that
the angles satisfy
$$
       |\cos\theta_1|\leq|\cos\theta_2|\leq\cdots\leq|\cos\theta_k|\leq|\cos\theta_{k+1}|\leq\cdots\leq|\cos\theta_n|,
$$
or
$$
          \frac{|\langle r_{0},g_1\rangle_m|}{\|r_{0}\|_m}
          \geq
          \cdots
          \geq
          \frac{|\langle r_{k-1},g_k\rangle_m|}{\|r_{k-1}\|_m}
          \geq
          \cdots
          \geq
          \frac{|\langle r_{n-1},g_n\rangle_m|}{\|r_{n-1}\|_m},
$$
since $\frac{|\langle
r_{k-1},g_k\rangle_m|}{\|r_{k-1}\|_m}=|\cos\theta_k|$. If the
algorithm  stops at the $k$-th iteration, then there is a $\delta
\in [|\cos\theta_k|,|\cos\theta_{k+1}|]$, which quantifies whether
an atom should be utilized to construct the final estimator. To be
detailed, if $|\cos{\theta_k}|\geq\delta$, then $g_k$ is regarded as
an ``active atom''
 and can be employed to build the estimator,
otherwise, $g_k$ is a ``dead one '' which should be deported.

Based on the above observations, we are interested in selecting
 arbitrary ``active atom'', $g_k$, in $\mathcal D_n$, that is
\begin{equation}\label{our metric}
           \frac{|\langle r_{k-1},g_k\rangle_m|}{\|r_{k-1}\|_m}\geq
           \delta.
\end{equation}
 If there is no $g_k$ satisfying (\ref{our metric}),
then the algorithm terminates. We call the greedy metric (\ref{our
metric}) as the ``$\delta$-greedy thresholds'' metric. In practice,
the number of ``active atom'' is usually not unique. Under this
circumstance, we can choose arbitrary (just) one ``active atom'' at
each greedy iteration. Once the ``active atom'' is selected, then
the algorithm  comes into the next greedy iteration and the ``active
atom'' is redefined. Through such a greedy-metric, we can develop a
new orthogonal greedy learning scheme, called thresholding
orthogonal greedy learning (TOGL). Instead of (II) and (IV) in OGL,
the corresponding parts of TOGL are described as follows

{\bf (II.1)  Greedy-definition:} Let $g_k$ be an arbitrary atom from
$\mathcal D_n$ satisfying
$$
              \frac{|\langle r_{k-1},g_k\rangle_m|}{\|r_{k-1}\|_m}\geq
           \delta.
$$

{\bf (IV.1)  Stopping criterion:} Terminate  the learning process
either there is not atom satisfying (\ref{our metric}) or $k$
satisfies a certain assumption.

Before giving the theoretical analysis of TOGL, we should highlight
the difference between (II), (IV) and (II.1), (IV.1), respectively.
Without considering the termination-rule, the classical greedy
metric (II) satisfies (II.1) since (II) always selects the greediest
atom in each greedy iteration. (II.1) slows down the  speed of
gradient descent and therefore may conduct a more flexible
model-selection strategy. According to the bias and variance balance
principle \cite{Cucker2007},
 the bias decreases while the variance
increases   as a new atom is selected to build the estimator.  If a
lower-correlation atom  is added, then the bias decreases slower and
the variance also increases slower. Then, the  balance  can be
achieved in TOGL   within a  more gradually flavor than OGL.
Compared with (IV), (IV.1) provides another termination condition
that if all the atoms, $g$, in $\mathcal D_n$ satisfy
\begin{equation}\label{Stop 1}
              \frac{|\langle r_{k-1},g\rangle_m|}{\|r_{k-1}\|_m}<
           \delta,
\end{equation}
then the algorithm terminates.  Programmers have asked us frequently
why there is the requirement of termination concerning $k$ besides
(\ref{Stop 1}), since their practical experience implies that the
termination condition (\ref{Stop 1}) is sufficient.  We emphasize
that the terminal condition concerning $k$ is necessary in TOGL, as
the numerical simulations usually do not face the worst case.
Indeed, using only the stopping  condition (\ref{Stop 1}) may drive
the algorithm to select all atoms from $\mathcal D_n$. For example,
if the target function $f$ is almost orthogonal to the space spanned
by the dictionary and the atoms in the dictionary  are almost linear
dependent (See Fig.3), then the selected $\delta$ should be very
small and such a small $\delta$ can not distinguish which   is the
``active atom ''. Consequently, the corresponding learning scheme
selects all the atoms of dictionary and therefore degrades the
generalization capability of OGL.
\begin{figure}[H]
\centering
\includegraphics[height=5cm,width=5.0cm]{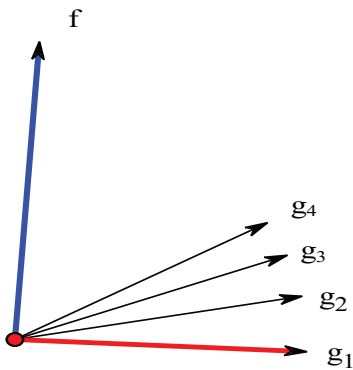}
\caption{Flaw of the single stopping condition}
\end{figure}

Now we present a theoretical assessment of TOGL.
At first, we give a few notations and concepts, which will be used throughout the paper.
%A few notations
%and concepts should be given at first.
Let $\mathcal L_1(\mathcal
D_n):=\{f:f=\sum_{g\in \mathcal D_n}a_gg\}$ endowed with the norm
$\|f\|_{\mathcal L_1(\mathcal D_n)}:=\inf\left\{\sum_{g\in \mathcal
D_n}|a_g|:f=\sum_{g\in \mathcal D_n}a_gg\right\}.$ For $r>0$, the
space $\mathcal L_1^r$ is defined to be the set of all functions $f$
such that,  there exists $h\in\mbox{span}\{\mathcal D_n\}$ such that

\begin{equation}\label{prior}
           \|h\|_{\mathcal L_1(\mathcal D_n)}\leq\mathcal B, \ \mbox{and}\
           \|f - h\| \leq {\mathcal  B}{n^{ - r}},
\end{equation}
where $\|\cdot\|$ denotes the uniform norm for the continuous
function space $C(X)$. The infimum of all such $\mathcal B$ defines
a norm (for $f$ ) on $\mathcal L_1^r$. It follows from
\cite{Barron2008} that  (\ref{prior}) defines a interpolation space
and is a natural assumption for the regression function in greedy
learning. Indeed, this assumption has already been adopted   in
\cite{Barron2008,Lin2013a} to analyze the learning capability of
greedy learning. The following Theorem \ref{THEOREM1} illustrate the
performance of TOGL and consequently, reveals the feasibility of the
greedy-metric (II.1).

\begin{theorem}\label{THEOREM1}
Let $0<t<1$, $0<\delta\leq 1/2$,  and $f_{\bf z}^{k,\delta}$  be the
estimator deduced by TOGL. If $f_\rho\in \mathcal L_1^r$, then there
exits a ${k^*} \in \mathbf N$ such that
$$
            {\cal E}({\pi _M}f_{\bf{z}}^{{k^*},\delta}) - {\cal E}({f_\rho }
            ) \le C{{\cal B}^2}({(m{\delta ^2})^{ - 1}}\log m\log\frac{1}{\delta }\log\frac{2}{t}
            + {\delta ^2} + {n^{ - 2r}})
$$
holds  with probability at least $1-t$, where $C$ is a positive
constant depending only on $d$ and $M$.
\end{theorem}

  If $\delta= \mathcal O
(m^{-1/4})$, and the size of dictionary, $n$, is selected to be
large enough, i.e., $n \geq \mathcal O({m^{\frac{1}{{4r}}}})$, then
our result shows that the generalization error bound of ${\pi
_M}f_{\bf{z}}^{{k^*},\delta}$ is asymptotically $\mathcal O
(m^{-1/2}(\log m)^2)$. Up to a logarithmic factor, this bound is the
same as that in \cite{Barron2008} and is the ``record'' of OGL. This
implies that weakening the level of greed of OGL within a certain
extent is a feasible way to circumvent the model selection problem
of OGL. It should also be  pointed out that different from OGL
\cite{Barron2008}, there are two parameters, $k$ and $\delta$, in
TOGL. Therefore, Theorem \ref{THEOREM1} only presents a theoretical
verification that introducing the ``$\delta$-greedy thresholds'' to
measure the level of greed does not essentially degrade the
generalization capability of OGL. Taking the practical applications
into account, eliminating the condition concerning $k$ in (IV.1)  is
urgent. This is the scope of the following section, where an
adaptive stopping criterion with respect to $\delta$ is presented.

\section{$\delta$-thresholding orthogonal greedy learning}

In TOGL, besides the greedy threshold parameter $\delta$, the
stopping criterion should be also adjusted appropriately, which may
 dampen the users' spirits to employ it. To circumvent this, in this section, we will develop an adaptive
stopping criterion based on the ``$\delta$-greedy thresholds''
metric. With this, we can develop a  practically user-friendly
orthogonal greedy type learning system.

It has been pointed out in the previous section that the reason of
employing the terminal condition concerning $k$ in (IV.1) is to
circumvent the extreme case for a full running of TOGL. As the high
impact atoms are all selected in such a setting,  they then lead the
relative value of the residual, $\| r_{k-1}\|_m/\|y(\cdot)\|_m$, to
be small, where $y(\cdot)$ is a function satisfies $y(x_i)=y_i,
i=1,\dots,m$. Therefore, a preferable terminal condition is to
quantify this relative value. Noting that $\delta$ has already been
utilized to terminate the algorithm, we append another  terminal
condition as
\begin{equation}\label{Our metric2}
          \| r_{k-1}\|_m \leq \delta\|y(\cdot)\|_m
\end{equation}
to replace the condition concerning $k$ in (IV.1). Based to this,
we obtain a novel applicable learning system by using the following
(IV.2) to substitute (IV.1) in TOGL.

{\bf (IV.2)  Stopping criterion:} Terminate  the learning process if
either (\ref{Our metric2}) holds or there is no atom satisfying
(\ref{our metric}).

\begin{algorithm}[H]\caption{$\delta$-TOGL}
\begin{algorithmic}
\STATE {{ Step 1 (Initialization)}: Given data ${\bf z}=
(x_i,y_i)_{i=1}^m$, dictionary $\mathcal D_n$, the greedy thresholds
$\delta$,
 and  $f_0=0$. Let $k:=0$.}
\STATE{ { Step 2 ($\delta$-greedy thresholds)}: Let $g_k$ be an
arbitrary atom from $\mathcal D_n$ satisfying
$$
              \frac{|\langle r_{k-1},g_k\rangle_m|}{\|r_{k-1}\|_m}\geq
           \delta.
$$
}
 \STATE{{Step 3 (Orthogonal projection iteration)}:
  Let
 $V_{{\bf z},k} =\mbox{Span}\{g_1,\dots,g_{k}\}$. Compute the approximation $f^{k,\delta}_{{\bf z}}$ as:
$$
  {f^{k,\delta}_{{\bf z}}} = {P_{{\bf z},V_{{\bf z},k}}}({  y})
$$
and the residual:
$$r_{k}:=y-f^{\delta,k}_{{\bf z}},$$
 where $P_{{\bf z},V_{{\bf z},k}}$ is the orthogonal projection  onto
space $V_{{\bf z},k}$ in the metric of
$\langle\cdot,\cdot\rangle_m$.}
 \STATE{ { Step 4
(Iteration)}: if
$$
  \max_{g\in \mathcal D_n}|\langle r_{k},g\rangle_m|\leq\delta\|r_k\|_{m} \ \text{or} \
  \|r_k\|_m\leq\delta\|f\|_m,
$$
then the algorithm terminates, otherwise let $k:=k+1$, and we turn
to Step 2.}
 \STATE{ {Output}: Since the stopping criterion depends
only on $\delta$,   we can write the final estimator as $f_{\bf
z}^\delta$.}
\end{algorithmic}
\end{algorithm}

For such  a setting, we succeed in avoiding the cumbersome parameter
$k$ and derive a stopping-criterion based only on $\delta$.  That
is, the main parameter $k$ of OGL \cite{Barron2008} is replaced by
the greedy thresholds $\delta$. Eventually, by utilizing the
``$\delta$-greedy thresholds'' metric and its corresponding adaptive
terminal rule (IV.2), we design a new learning system called
$\delta$-thresholding orthogonal greedy learning ($\delta$-TOGL) as
in the   Algorithm 1.

The following Theorem \ref{THEOREM2} shows
that if $\delta$ is appropriately tuned, then the $\delta$-TOGL
estimator $f_{\bf z}^\delta$ can realize the almost optimal
generalization capability of OGL and TOGL.

\begin{theorem}\label{THEOREM2}
Let $0<t<1$, $0<\delta\leq 1/2$,  and $f_{\bf z}^\delta$  be defined
in Algorithm 1. If $f_\rho\in \mathcal L_1^r$, then the inequality
\begin{equation}\label{Theorem2}
              \mathcal E(\pi_Mf_{\bf z}^\delta)-\mathcal E(f_\rho)
             \leq
             C{{\cal B}^2}({(m{\delta ^2})^{ - 1}}\log m
             \log\frac{1}{\delta }\log\frac{2}{t} + {\delta ^2} + {n^{ - 2r}})
\end{equation}
holds  with probability at least $1-t$, where $C$ is a positive
constant depending only on $d$ and $M$.
\end{theorem}

If we choose $n \geq \mathcal O({m^{\frac{1}{{4r}}}})$  and $\delta=
\mathcal O (m^{-1/4})$, then the learning rate of  (\ref{Theorem2})
asymptotically equals to $\mathcal O (m^{-1/2}(\log m)^2)$, which is
the same as that of Theorem \ref{THEOREM1}. Therefore, Theorem
\ref{THEOREM2} implies that using (\ref{Our metric2}) to replace the
terminal condition concerning $k$ in (IV.1) is theoretically
feasible. From the viewpoint of implementation, the stopping
criterion (IV.2) is  far more user-friendly  than that of (IV.1),
since (IV.2) omits the parameter $k$ of (IV.1) without scarifying
the generalization capability of TOGL.

The most highlight of Theorem \ref{THEOREM2} is that it provides a
totally different way to circumvent the overfitting phenomenon of
OGL. It is known  that the stopping criterion is crucial for OGL,
but designing an effective  stopping criterion   is a awkward
problem. Barron et al. \cite{Barron2008} suggested to select $k$
that minimizes a $l^0$ based complexity regularization strategy,
which often needs a full running before the best parameter is
selected. Chen et al. \cite{Chen2013} proposed a stopping criterion
also leads to a long iterative procedure in practice and sometimes
does not work.   In short, all the aforementioned study of
stopping-criterion attempted to design a terminal rule by
controlling the number of iterations directly. Since the
generalization capability of OGL is  sensitive to the number of
iterations, these schemes sometimes fails to   get satisfactory
effects. The terminal rule employed in the present paper is based on
the study of the ``greedy-metric'' issue of greedy learning. Theorem
\ref{THEOREM2} shows that, besides controlling the number of
iterations directly, setting a greedy threshold to redefine the
greed can also conducts an effective stopping criterion. Theorem
\ref{THEOREM2} implies that this new stopping criterion
theoretically works as well as others. Furthermore, when compared
with $k$ in OGL, the generalization capability of the $\delta$-TOGL
is stable to $\delta$,  since the new metric slows down the changes
of bias and variance.

\section{Numerical Studies}

In this section, we present several numerical simulations to reveal
the pros and cons of  $\delta$-TOGL. We divide the description into
seven subsections. Except for the first one, each subsection depicts
a topic concerning $\delta$-TOGL.

\subsection{Experimental settings and purpose}

\underline{Data and dictionary:} The samples ${\bf
z}=\{(x_{i},y_{i})\}_{i=1}^{m_1}$ are generated as follows.
$\{x_i\}_{i=1}^{m_1}$ are drawn independently and identically
according to the uniform distribution on $[-\pi,\pi]$.
$\{y_i\}_{i=1}^{m_1}$ satisfies $y_{i}=f_{\rho }(x_{i})+\mathcal
N(0,\sigma^2)  $ with $\mathcal  N(0,\sigma^2)$ being the white
noise and
$$
                {f_\rho}(x) = \frac{{\sin x}}{x}, \quad x \in [ - \pi ,\pi
                ].
$$
 To comprehensively reveal the
performances of OGL, TOGL and $\delta$-TOGL, we adopt four levels of
noise, that is, $\sigma$ is set  to $\sigma_1=0.1$, $\sigma_2=0.5$,
$\sigma_3=1$ and $\sigma_4=2$. The learning performances of
different algorithms
were then tested by applying the resultant estimators to the test set ${\bf z}_{test}=%
\{(x_{i}^{(t)},y_{i}^{(t)})\}_{i=1}^{m_2}$, which was generated
similarly to ${\bf z}$ but with a promise that $y_{i}^{\prime }s$ were always taken to be $%
y_{i}^{(t)}=f_{\rho }(x_{i}^{(t)}).$

In each simulation,  we use Gaussian radial basis function to build
up the dictionary:
$$
            \left\{e^{-\|x-t_i\|^2/\eta^2}: i=1, \ldots,n\right\},
$$
where $\{t_i\}_{i=1}^n$ are drawn as the best packing points in
$[-\pi,\pi]$. Since, the aim of the simulations is not to  pursue
the best width of Gaussian radial basis function, but to compare
$\delta$-TOGL with other learning schemes on the same dictionary,
 we always  set $\eta=1$ throughout this section.

\underline{Methods:} For OGL and $\delta$-TOGL, we apply the QR
decomposition to solve the corresponding least squares problem and
then obtain the estimators \cite{Sauer2006}. We use four metrics in
(II) and (II.1) respectively to illustrate different levels of
greed. Here, we use abbreviations OGL1 ,  OGL2,   OGL3, TOGL1,
TOGL2, TOGL3,   and $\delta$-TOGL1, $\delta$-TOGL2 ,
 $\delta$-TOGL3  to denote OGL, TOGL  and $\delta$-TOGL with  (II), and (II.1)
 replaced by
$$
  {g_k}: = \arg \mathop {\max }\limits_{g \in \mathcal D_n} | \langle {r_{k - 1}},g \rangle_m
  |,
$$
$$
  {g_k}: = \arg \mathop {second \max }\limits_{g \in \mathcal D_n} | \langle {r_{k - 1}},g \rangle_m
  |,
$$
and
$$
  {g_k}: = \arg \mathop {third \max }\limits_{g \in \mathcal D_n} | \langle {r_{k - 1}},g\rangle_m |.
$$
Here, $\arg \mathop {second \max }\limits_{g \in \mathcal D_n}$ and
$\arg \mathop {third \max }\limits_{g \in \mathcal D_n}$ means
selecting $g_k$ such that the second and third largest values of
$|\langle r_{k-1},g\rangle_m|$ are attained, respectively.
Furthermore, we use OGLR, TOGLR and $\delta$-TOGLR to denote OGL,
TOGL, and $\delta$-TOGL with (II) and (II.1) replaced by
$$
  {g_k} \mbox{ randomly selected from}\
  \mathcal{D}_n ,
$$
and
$$
  {g_k} \mbox{   randomly   selectd from}\
  \mathcal{D_\delta} \ \mbox{with}\
  \mathcal D_\delta=\{g_j:\langle
  g_j,r_{k-1}\rangle_m\geq\delta\|r_{k-1}\|_m\}.
$$
  We also compare our methods with two widely used
learning schemes such as ridge regression \cite{Golub1979} and Lasso
\cite{Tibshirani1995}. We use the analytic solutions to  ridge
regression \cite{Golub1979} and implementing the fast iterative soft
thresholding algorithm (FISTA) \cite{Beck2009} for  Lasso to deduce
the corresponding estimators.

\underline{Aims of simulations} The aims of the simulations can be
concluded into six aspects.
 In Sec.5.2, we demonstrate that SGD is not the unique metric to define greed in OGL. Indeed,
our simulation shows that OGL2 and OGL3 possess almost the same
generalization capabilities as that of OGL1. In Sec.5.3, we
illustrate that ``$\delta$-greedy thresholds'' is a feasible greedy
metric. In Sec.5.4, we aim to provide numerical verification of
the good performance of  $\delta$-TOGL.   In Sec.5.5, we analyze how
the parameter $\delta$ affects the training time and the sparsity of
the estimator. In Sec.5.6, we conduct a phase-transition diagram to
illustrate the usability and limitations of $\delta$-TOGL. In
Sec.5.7, we compare $\delta$-TOGL with other widely used
dictionary-based learning schemes and then show the feasibility of
$\delta$-TOGL.

\underline{Environment:} All numerical studies are implemented by
MATLAB R2013a on a Windows personal computer with Core(TM) i7-3770
3.40GHz CPUs and RAM 4.00GB, and the statistics are averaged based
on 50 independent trails.

\subsection{Greedy metric of OGL}

In this part, we illustrate that SGD is not the unique metric for
OGL. To this end, we conduct  simulations for $f_\rho$ with the
aforementioned four types of noise. We sample $m_1=1000$ training
samples  and $m_2=1000$ testing samples. The number of centers is
set to $n=300$. Under this setting, we run 5  times of
simulations and describe its average test errors, which is measured
by the rooted mean square error (RMSE), as functions of the number
of iterations, $k$, of OGL1, OGL2, OGL3 and OGLR. Since the optimal
$k$ is small and the test RMSE is   very large when $k$ is large, we
only record the figures with $k\in[0,15]$. The experimental results
are shown in the following Fig.4.
\begin{figure}[H]
\centering
\subfigure[]{\label{Fig.sub.b}\includegraphics[height=5cm,width=6cm]{OGAd1.eps}}
\subfigure[]{\label{Fig.sub.b}\includegraphics[height=5cm,width=6cm]{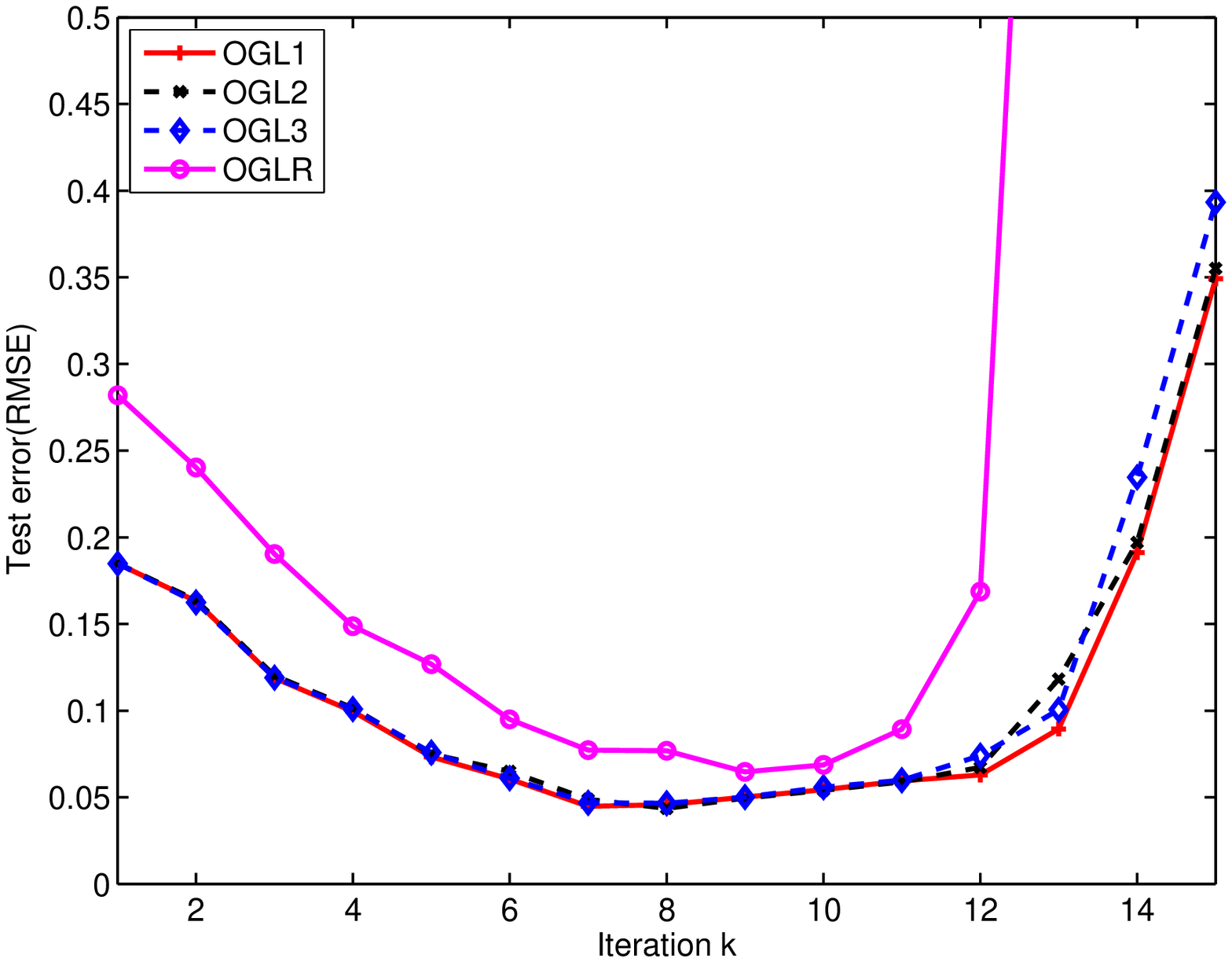}}
\subfigure[]{\label{Fig.sub.b}\includegraphics[height=5cm,width=6cm]{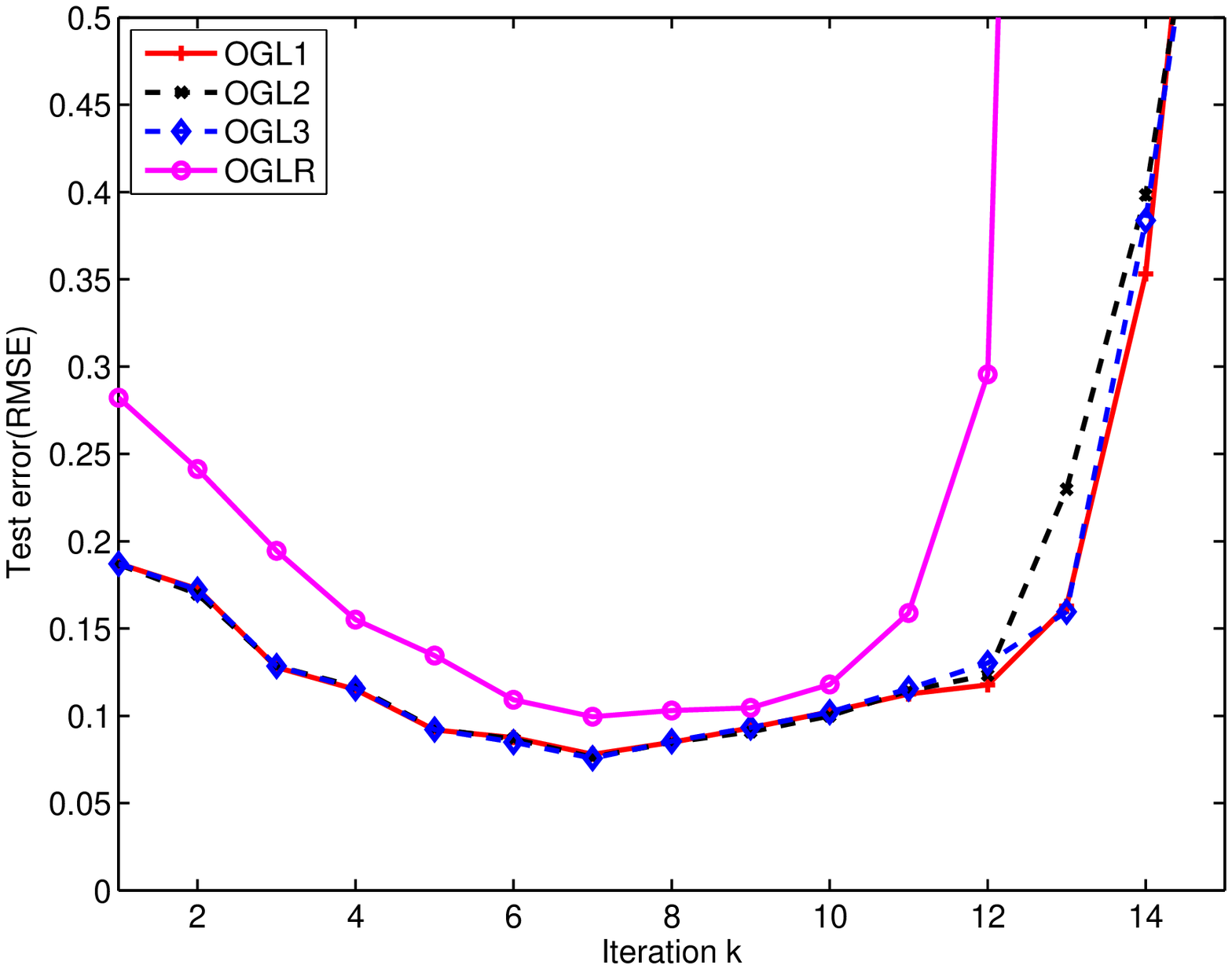}}
\subfigure[]{\label{Fig.sub.b}\includegraphics[height=5cm,width=6cm]{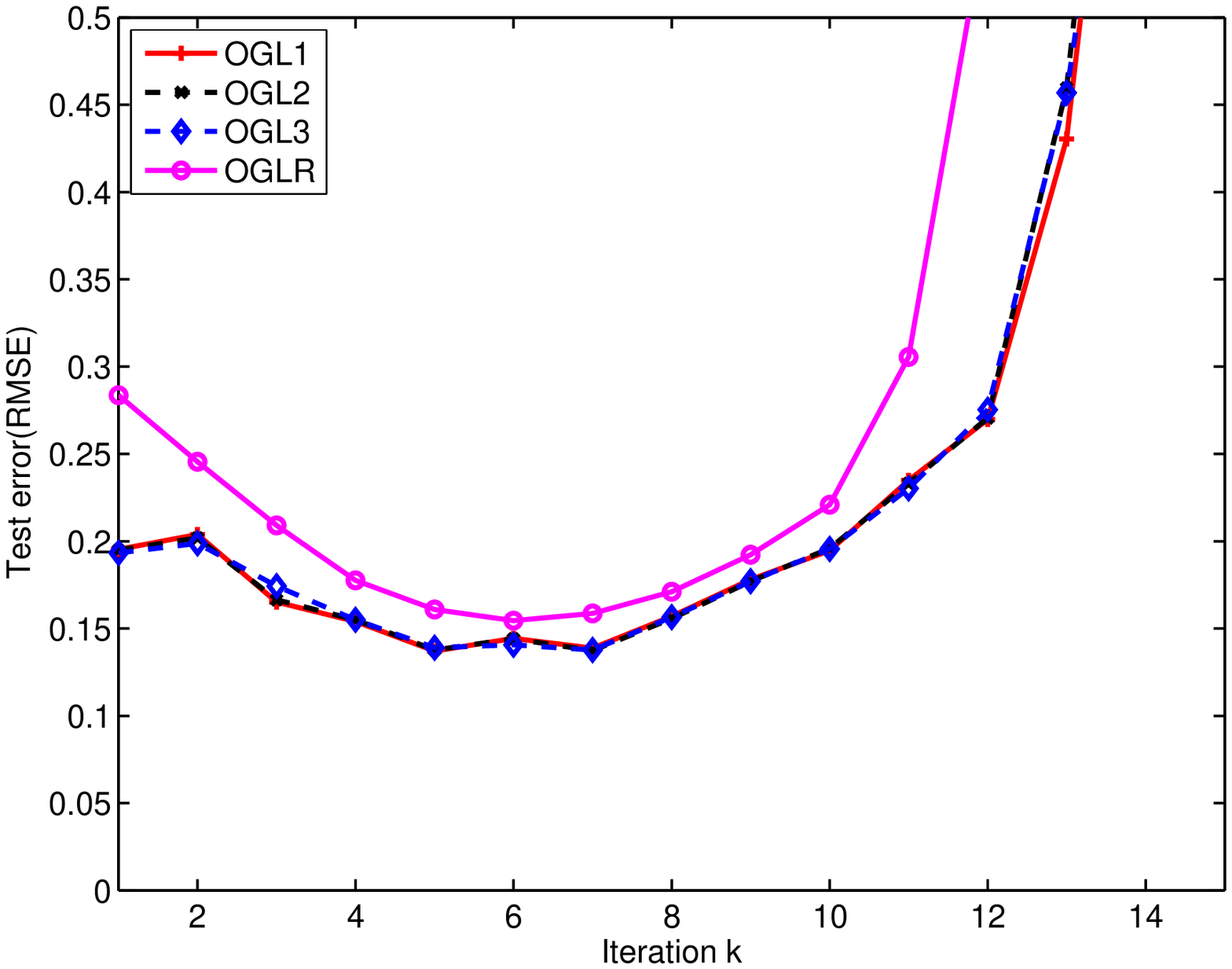}}
\caption{The generalization capabilities of OGL with different
greedy metrics}
\end{figure}

Fig.4 (a)-(d) shows the learning capabilities of  OGL for $f_\rho$
with different levels of noise from $\delta_1$ to $\delta_4$. It can
be found that OGL1, OGL2 and OGL3 possess almost the same
generalization capabilities, since both the smallest test RMSE and
the optimal $k$ of them are almost the same. This implies that, at
least for a certain learning task, SGD is not the unique metric for
OGL. Furthermore, it can also be found in Fig.4 that OGLR performs
worse than that of other learning schemes. This phenomenon shows
that introducing a greedy metric is necessary. We also give a
quantitive comparison of the learning performances of OGL1, OGL2,
OGL3, and OGLR in the following Tab.1. Here \text{$TestRMSE_{OGL}$}
and $k_{OGL}^*$ denote the theoretically optimal test RMSEs and $k$
  of OGL with different greedy metrics.
 Indeed, $k^*_{OGL}$'s are selected according to the test data directly.

\begin{table}[H] \renewcommand{\arraystretch}{0.7}
\begin{center}
 \caption{OGL numerical average results for 5 simulations.}\label{table3}
\begin{minipage}{0.45\textwidth}
 \begin{tabular}{|c|c|c|}\hline
%   \multicolumn{3}{|c|}{$sinc \ function $} \\ \hline
 $Methods$ & $TestRMSE_{OGL}$  &${k_{OGL}^*}$   \\ \hline
  \multicolumn{3}{|c|}{$\sigma=0.1$} \\ \hline
 OGL1   &0.0249  &9   \\ \hline
 OGL2   &0.0248  &9    \\ \hline
 OGL3   &0.0251  &10  \\ \hline
 OGLR   &0.0304  &9  \\ \hline

  \multicolumn{3}{|c|}{$\sigma=0.5$} \\ \hline
 OGL1   &0.0448  &7   \\ \hline
 OGL2   &0.0436  &8 \\ \hline
 OGL3   &0.0466  &8   \\ \hline
 OGLR   &0.0647  &9  \\ \hline
 \end{tabular}
 \end{minipage}
\begin{minipage}{0.45\textwidth}
 \begin{tabular}{|c|c|c|}\hline
 $Methods$ & $TestRMSE_{OGL}$  &${k_{OGL}^*}$   \\ \hline
 \multicolumn{3}{|c|}{$\sigma=1$} \\ \hline
 OGL1   &0.0780  &7   \\ \hline
 OGL2   &0.0762  &7 \\ \hline
 OGL3   &0.0757 &7   \\ \hline
 OGLR   &0.0995  &7  \\ \hline
 \multicolumn{3}{|c|}{$\sigma=2$} \\ \hline
 OGL1   &0.1371  &5   \\ \hline
 OGL2   &0.1374 &7 \\ \hline
 OGL3   &0.1377 &7   \\ \hline
 OGLR   &0.1545  &6  \\ \hline
 \end{tabular}
 \end{minipage}
 \end{center}
 \end{table}

All the above simulations show that greed is necessary    but not
unique in OGL. This stimulates us to launch a study of the
``greedy-metric'' issue of OGL.

\subsection{  ``$\delta$-greedy thresholds'' metric  }

In this part, we verify the feasibility of the ``$\delta$-greedy
thresholds'' metric proposed in Sec.3. The simulation setting of
this subsection is the same as that of Sec.5.2. We also run 5  times
of simulations and describe its test RMSE as functions of the
threshold, $\delta$, of TOGL1, TOGL2, TOGL3 and TOGLR, where  we
choose the optimal number of iterations based on the test set. There
are 100 candidates of $\delta$ which are equally logarithmically
drawn from $[10^{-6}, 1/2]$. Since the optimal value of $\delta$ lies in $[10^{-6},0.001]$, we only plot the range of
$\delta$ in $[10^{-6},0.001]$ to present more details of the
simulations. The experimental results are reported in the following
Fig.5.
\begin{figure}[H]
\centering
\subfigure[]{\label{Fig.sub.a}\includegraphics[height=5cm,width=6cm]{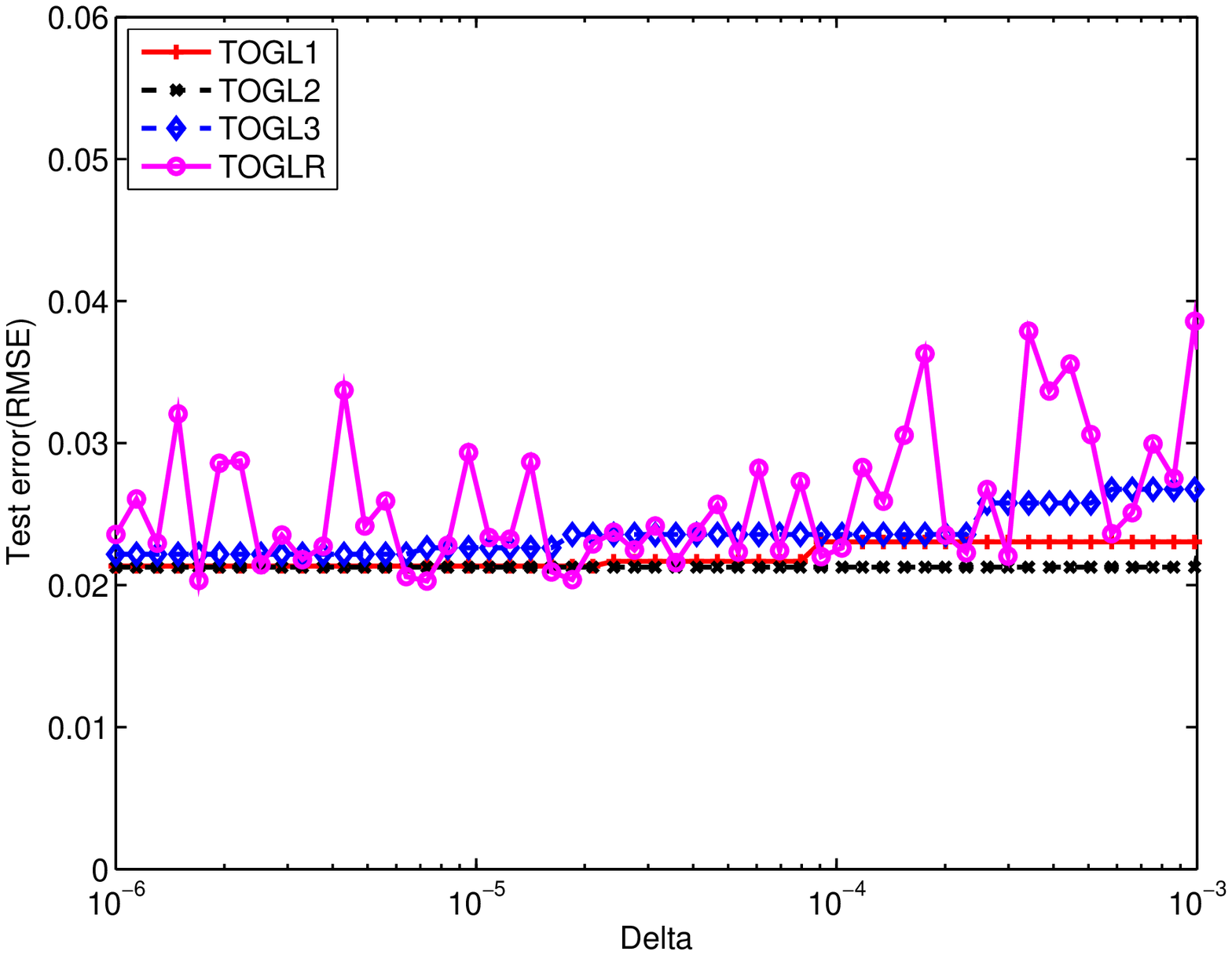}}
\subfigure[]{\label{Fig.sub.a}\includegraphics[height=5cm,width=6cm]{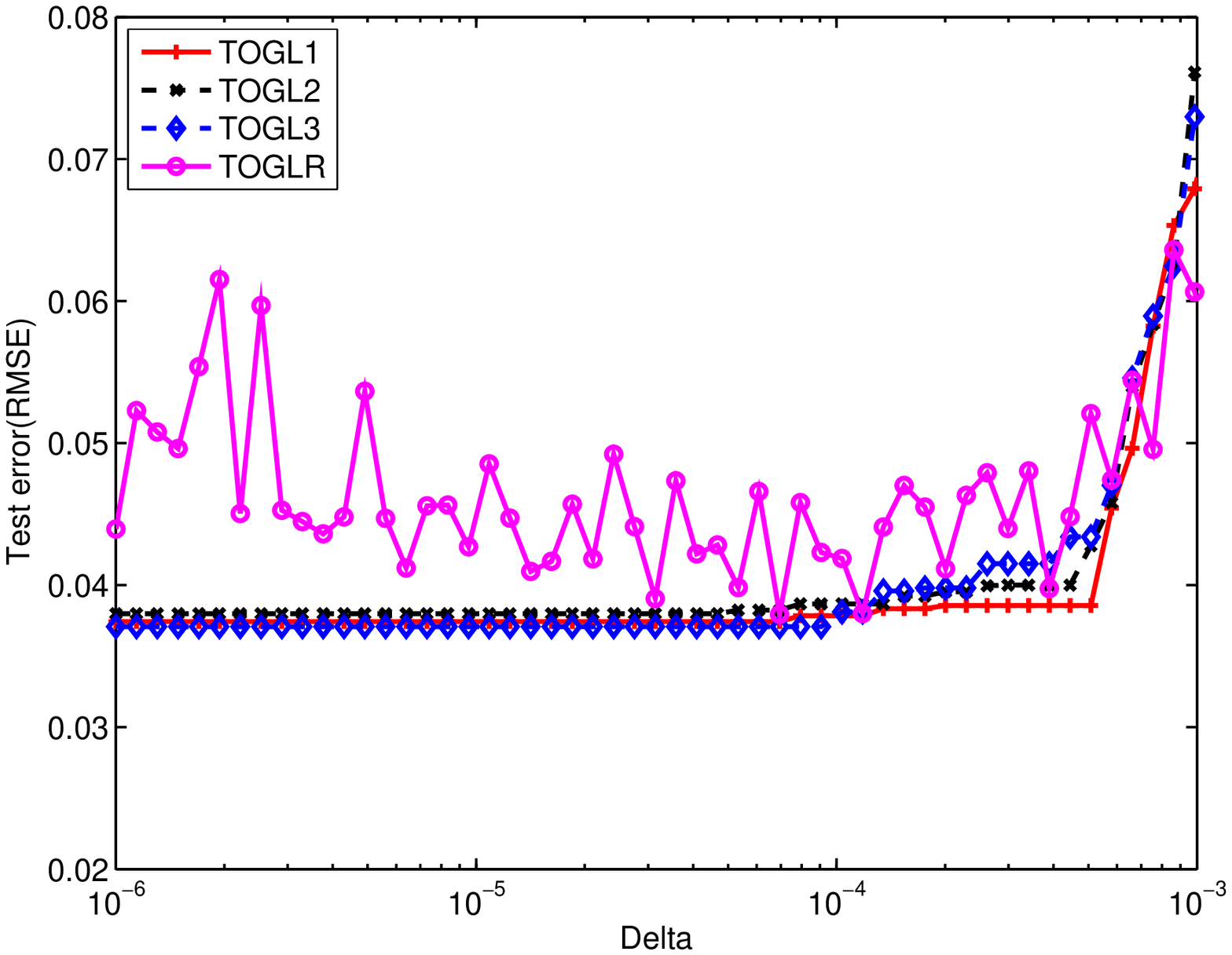}}
\subfigure[]{\label{Fig.sub.a}\includegraphics[height=5cm,width=6cm]{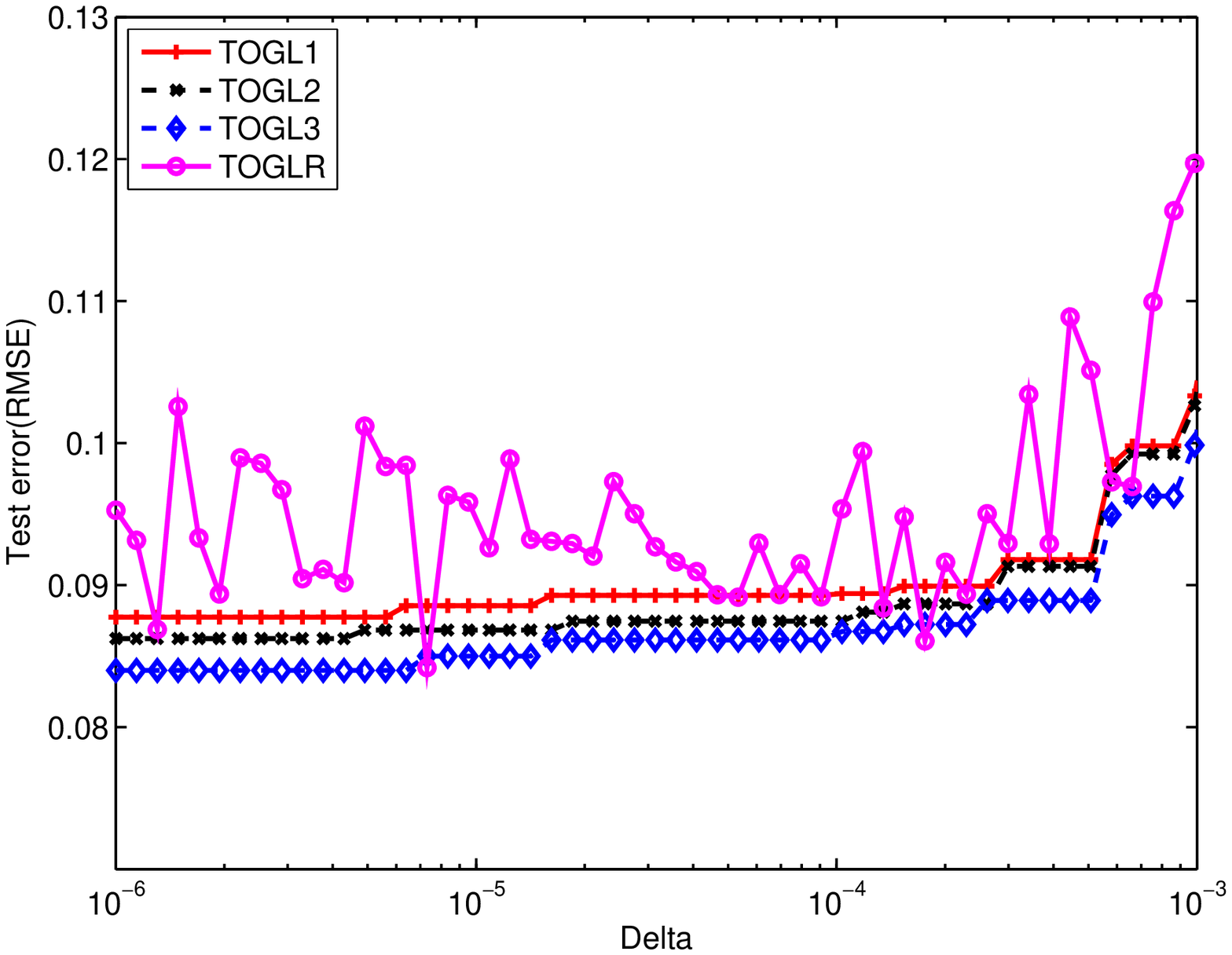}}
\subfigure[]{\label{Fig.sub.a}\includegraphics[height=5cm,width=6cm]{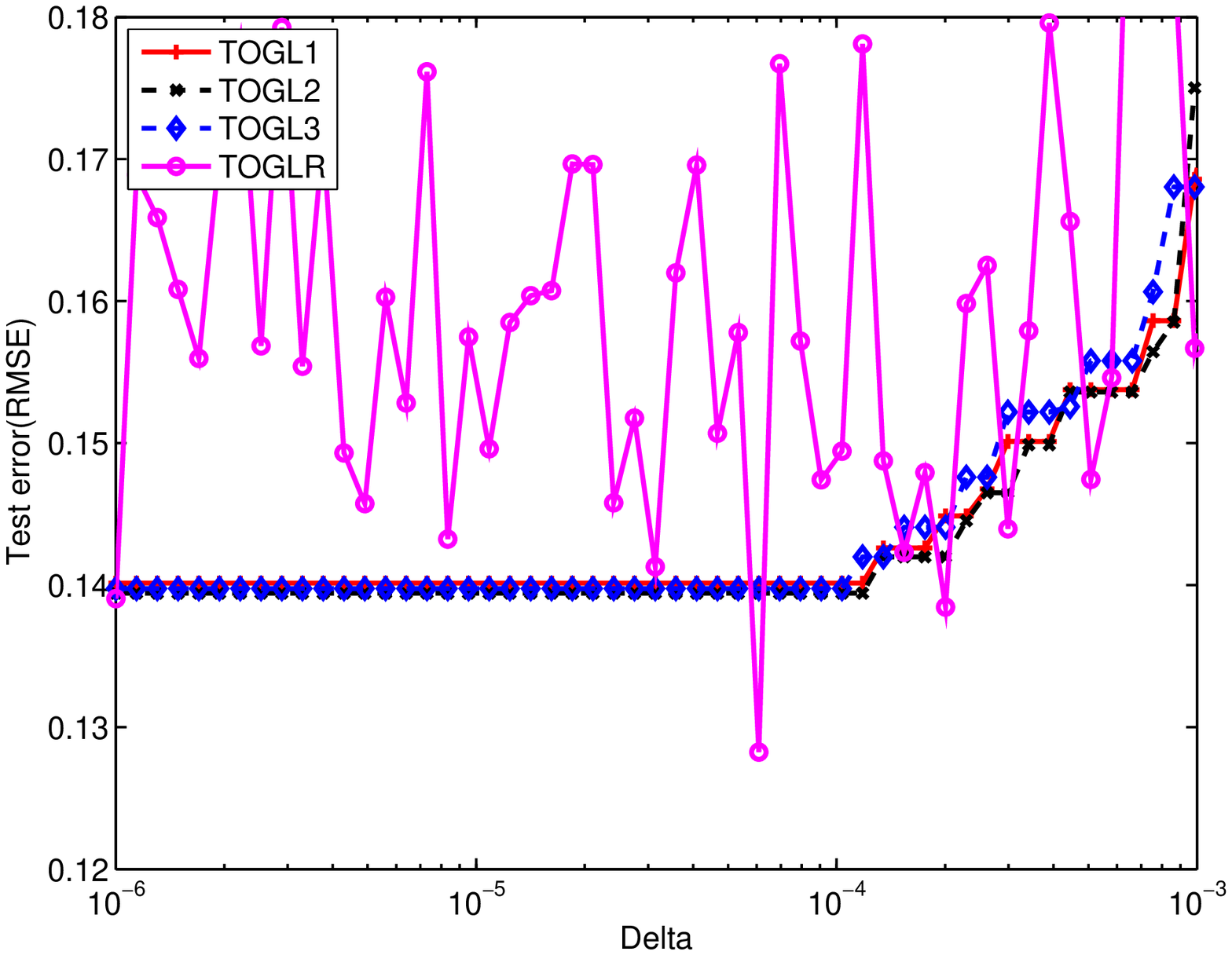}}
\caption{The feasibility of  the ``$\delta$-greedy threshold''
metric}
\end{figure}
Fig.5 shows that, different from Fig.4, the learning capability of
TOGLR is similar as that of TOGL1, TOGL2 and TOGL3. The main reason
is that we select the new atom (even for the random selected atom)
in a greedy fashion by adding the ``$\delta$-greedy thresholds''
metric in TOGL.
 This phenomenon
implies that once an appropriately   $\delta$ is preset, then how to
choose the atom according to (II.1) is not crucial. Therefore, it
numerically verifies Theorem \ref{THEOREM1} and demonstrates that
the introduced ``$\delta$-greedy threshold'' is feasible and
appropriate to quantify the greedy metric. To facilitate the
comparison, we also record the optimal generalization errors in
Tab.2.

In Tab.2, the second column (i.e., ``${\delta}$ and $k$'')  records
the optimal $\delta$ value and their corresponding $k$ values (in
the bracket) derived from TOGL. We should highlight that these $k$
are obtained by using the   terminal condition (\ref{Stop 1}) only.
We also use ${k_{TOGL}^*}$   to denote the theoretically optimal $k$
of TOGL, which is selected based on the test set. It can be found in
Tab.2 that when $\delta$ equals to $0.1$ or $0.5$, the corresponding
$k$ is almost the same as ${k_{TOGL}^*}$, which means that using the
terminal condition (\ref{Stop 1}) is sufficient to select the
optimal iteration number. However, if the noise is enlarged, that
is, $\delta=1$ or $2$, then the terminal condition (\ref{Stop 1})
usually fails to find out the optimal $k$ and another stopping
condition need to be employed. This explains why we introduce a
terminal condition concerning $k$ in (IV.1) and an adaptive terminal
condition (\ref{Our metric2}) in (IV.2). Compared with Tab.1,
%Compared Tab.2 with Tab.1,
we can find from Tab.2 that the optimal test RMSEs
($TestRMSE_{TOGL}$ and $TestRMSE_{OGL}$) are comparable,  which
illustrates that the ``$\delta$-greedy thresholds'' metric is
feasible. The new greedy metric then provides an alternative way to
enrich the  model-selection strategy without scarifying the
generalization capability of OGL.

\begin{table}[H] \renewcommand{\arraystretch}{0.7}
\begin{center}
 \caption{TOGL numerical average results for 5 simulations.}\label{table4}
%\begin{minipage}{0.495\textwidth}
 \begin{tabular}{|c|c|c|c|}\hline

$ Methods$ & ${\delta}$ and $k$& $TestRMSE_{TOGL}$ &${k_{TOGL}^*}$            \\ \hline
  \multicolumn{4}{|c|}{$\sigma=0.1$} \\ \hline
 TOGL1   &[\text{1.00e-6,3.58e-5}]([9,13])      &0.0213  &8   \\ \hline
 TOGL2   &[\text{1.00e-6,1.70e-6}]([11,12])     &0.0213  &8    \\ \hline
 TOGL3   &[\text{1.00e-6,1.70e-6}]([12,13])     &0.0222  &10   \\ \hline
 TOGLR   &\text{9.52e-6}(12)                    &0.0203  &11   \\ \hline

  \multicolumn{4}{|c|}{$\sigma=0.5$} \\ \hline
 TOGL1   &[\text{1.00e-6,6.95e-5}]([8,13])    &0.0374  &8   \\ \hline
 TOGL2   &[\text{1.00e-6,4.67e-5}]([9,13])    &0.0380  &8   \\ \hline
 TOGL3   &[\text{1.00e-6,9.06e-5}]([8,13])    &0.0371  &8   \\ \hline
 TOGLR   &\text{6.95e-5}(9)                   &0.0379  &8   \\ \hline

 \multicolumn{4}{|c|}{$\sigma=1$} \\ \hline
 TOGL1   &[\text{1.00e-6,5.60e-6}]([11,13]) &0.0877  &8  \\ \hline
 TOGL2   &[\text{1.00e-6,4.30e-6}]([11,13]) &0.0862  &8  \\ \hline
 TOGL3   &[\text{1.00e-6,6.40e-6}]([11,13]) &0.0840  &8  \\ \hline
 TOGLR   &\text{7.30e-6}(12)                &0.0842  &8  \\ \hline

 \multicolumn{4}{|c|}{$\sigma=2$} \\ \hline
 TOGL1   &[\text{1.00e-6,1.18e-4}]([8,13])  &0.1402  &6   \\ \hline
 TOGL2   &[\text{1.00e-6,1.18e-4}]([8,13])   &0.1394  &6  \\ \hline
 TOGL3   &[\text{1.00e-6,1.03e-4}]([8,13]) &0.1398  &6      \\ \hline
 TOGLR   &\text{6.09e-5}(10)                  &0.1282  &5  \\ \hline

 \end{tabular}
 %\end{minipage}
 \end{center}
 \end{table}

\subsection{The generalization capability of  $\delta$-TOGL}

In this part, we justify the good performance of $\delta$-TOGL
proposed in Sec.4. The detailed experimental setting is the same as
that in Sec.5.3. Different from TOGL, $\delta$-TOGL provides an
adaptive terminal rule and therefore, eliminates the parameter $k$
in TOGL. Similarly to Sec.5.3, we only plot the range of $\delta$ in
$[10^{-6},0.001]$ to reveal more details of the simulations. The
following Fig.6 reports the simulations results.
\begin{figure}[H]
\centering
\subfigure[]{\label{Fig.sub.a}\includegraphics[height=5cm,width=6cm]{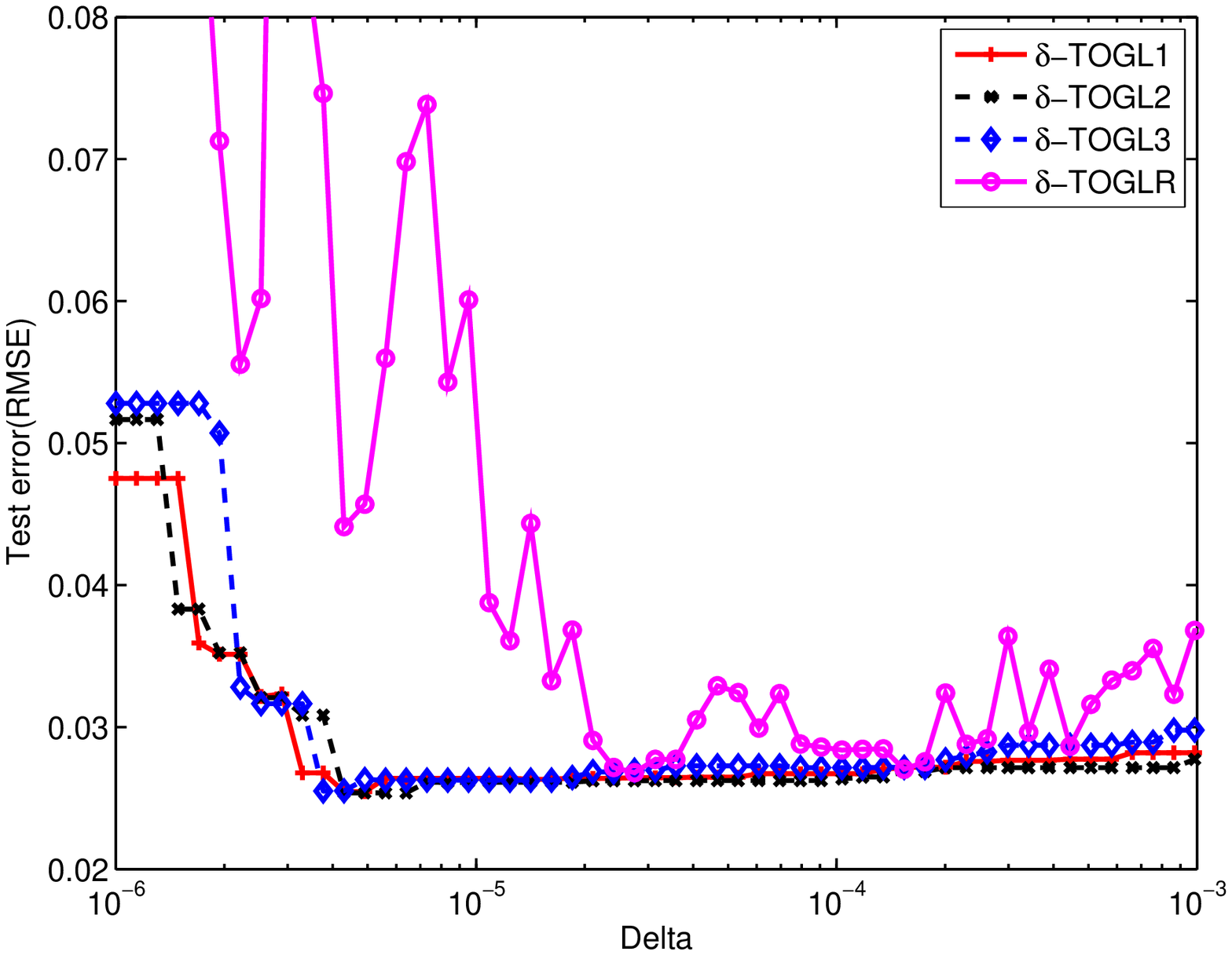}}
\subfigure[]{\label{Fig.sub.a}\includegraphics[height=5cm,width=6cm]{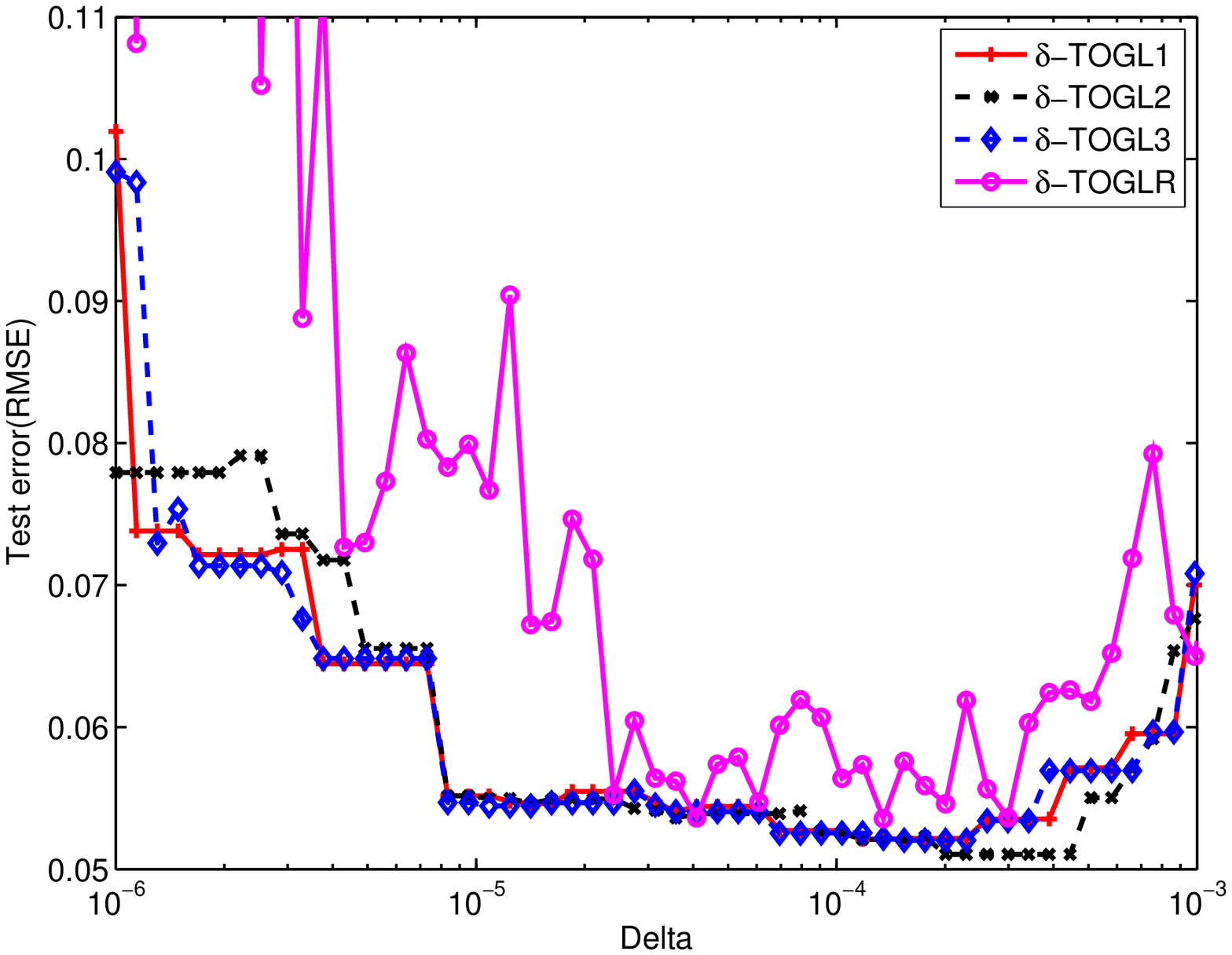}}
\subfigure[]{\label{Fig.sub.a}\includegraphics[height=5cm,width=6cm]{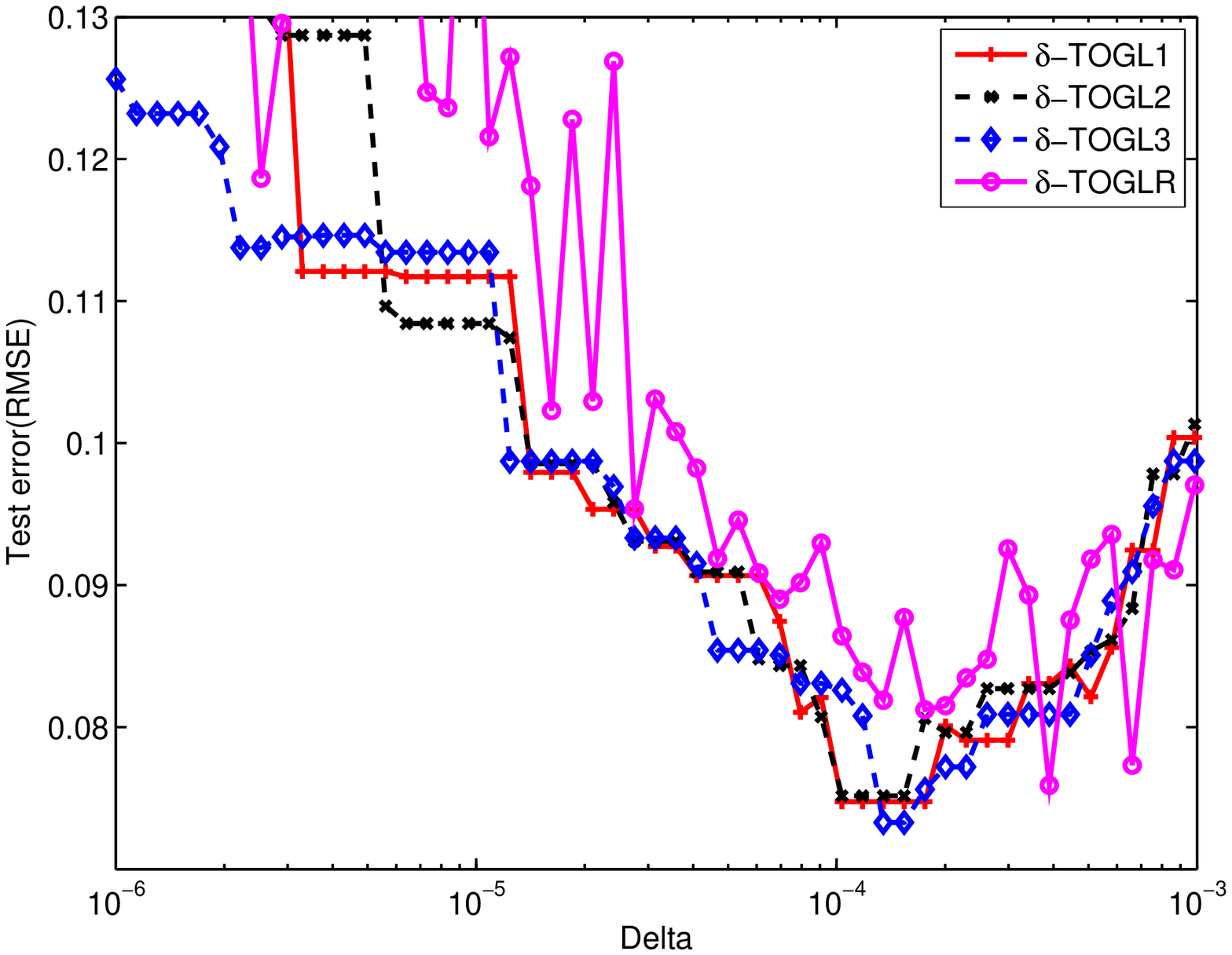}}
\subfigure[]{\label{Fig.sub.a}\includegraphics[height=5cm,width=6cm]{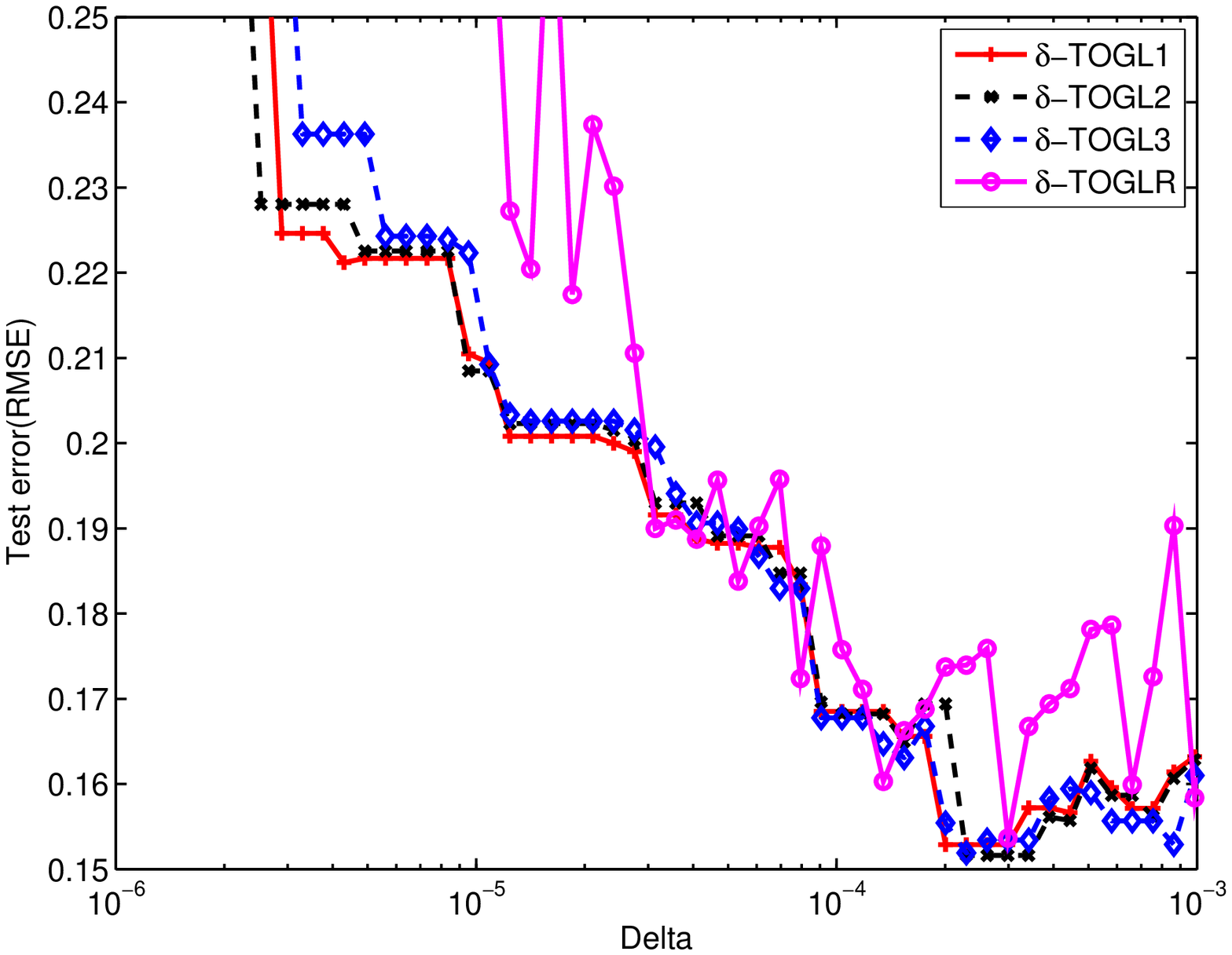}}
\caption{The feasibility of  the $\delta$-TOGL}
\end{figure}

%\begin{figure}[H]
%\centering
%\subfigure[]{\label{Fig.sub.a}\includegraphics[height=5.5cm,width=5.5cm]{figuress/DTOGAkd1.eps}}
%\subfigure[]{\label{Fig.sub.a}\includegraphics[height=5.5cm,width=5.5cm]{figuress/DTOGAkd3.eps}}
%\subfigure[]{\label{Fig.sub.a}\includegraphics[height=5.5cm,width=5.5cm]{figuress/DTOGAkd4.eps}}
%\subfigure[]{\label{Fig.sub.a}\includegraphics[height=5.5cm,width=5.5cm]{figuress/DTOGAkd5.eps}}
%\caption{The generalization capability of $\delta$-TOGL}
%\end{figure}

Fig.6 shows that $\delta$-TOGL maintains the feasibility of
``$\delta$-greedy thresholds'' metric after introduced the adaptive
termination rule (IV.2). Therefore, it numerically verifies
Theorem \ref{THEOREM2} and demonstrates that $\delta$-TOGL is
feasible. We also show  the generalization capability of
$\delta$-TOGL  in the following Tab.3.

\begin{table}[H] \renewcommand{\arraystretch}{0.7}
\begin{center}
 \caption{$\delta$-TOGL numerical average results for 5 simulations.}
%\begin{minipage}{0.45\textwidth}
 \begin{tabular}{|c|c|c|c|}\hline
 $Methods$ & ${\delta}$ and $k$ & $TestRMSE_{\delta-TOGL} $ & ${k_{\delta-TOGL}^*}$       \\ \hline
  \multicolumn{4}{|c|}{$\sigma=0.1$} \\ \hline
 $\delta$-TOGL1    &[\text{4.30e-6,4.91e-6}](11)    &0.0255 &10.6   \\ \hline
 $\delta$-TOGL2    &[\text{5.60e-6,6.40e-6}](10.4)    &0.0254  &10.2  \\ \hline
 $\delta$-TOGL3    &\text{3.76e-6}(11)              &0.0255  &10.6   \\ \hline
 $\delta$-TOGLR    &\text{2.75e-5}(11)              &0.0268 &10.8  \\ \hline

 \multicolumn{4}{|c|}{$\sigma=0.5$} \\ \hline
 $\delta$-TOGL1    &[\text{1.18e-4,1.35e-4}](7.4)      &0.0521  &7.4    \\ \hline
 $\delta$-TOGL2    &[\text{2.01e-4,4.45e-4}](7)      &0.0511 &7   \\ \hline
 $\delta$-TOGL3    &[\text{1.54e-4.2.29e-4}](7.2)      &0.0520  &7.2   \\ \hline
 $\delta$-TOGLR    &\text{1.35e-4}(8.6)                &0.0536  &8.6   \\ \hline

 \multicolumn{4}{|c|}{$\sigma=1$} \\ \hline
 $\delta$-TOGL1    &[\text{1.03e-4,1.76e-4}](7.2)    &0.0747  &6.8     \\ \hline
 $\delta$-TOGL2    &[\text{1.03e-4,1.54e-4}](7.2)    &0.0752  &6.8   \\ \hline
 $\delta$-TOGL3    &[\text{1.35e-4,1.54e-4}](7.2)    &0.0733  &7   \\ \hline
 $\delta$-TOGLR    &\text{3.89e-4}(7.2)              &0.0759  &6.4    \\ \hline

 \multicolumn{4}{|c|}{$\sigma=2$} \\ \hline
 $\delta$-TOGL1    &[\text{2.01e-4,2.99e-4}](6.2)    &0.1529  &5.4   \\ \hline
 $\delta$-TOGL2    &[\text{2.29e-4,3.41e-4}](6.2)    &0.1516  &5.6   \\ \hline
 $\delta$-TOGL3    &\text{2.29e-4}(6.2)             &0.1519  &4.8   \\ \hline
 $\delta$-TOGLR    &\text{2.99e-4}(7.2)              &0.1537  &6.2   \\
 \hline
 \end{tabular}
 %\end{minipage}
 \end{center}
 \end{table}

In Tab.3, the second column (i.e., ``${\delta}$ and $k$'')  records
the optimal $\delta$ and the corresponding $k$ (in the bracket)
derived from $\delta$-TOGL, and ${k_{\delta-TOGL}^*}$ denotes the
theoretically optimal $k$ of $\delta$-TOGL. It can be found that for
all types of noise, $k$ is almost the same as ${k_{\delta-TOGL}^*}$.
This shows that the stopping condition concerning $k$ in (IV.1) can
be substituted with the terminal condition (\ref{Our metric2}).
Therefore, these experimental results demonstrate in some extent
that we can avoid the ``overfitting'' by only taking the
``greedy-metric'' issue into account. This can be regarded as the
main novelty of our paper. Furthermore, noting that the optimal test
RMSEs ($TestRMSE_{\delta-TOGL} $) are  comparable with
$TestRMSE_{TOGL} $,  we can declare that  $\delta$-TOGL performs as
well as TOGL, while $\delta$-TOGL successfully omit the parameter
concerning $k$
  in TOGL.

\subsection{The cost of alternating parameter of $\delta$-TOGL}

From OGL to $\delta$-TOGL, the main parameter is changed from $k$ to
$\delta$. In the previous subsections, we pointed out that the
generalization capability of such a change was not degraded.
Furthermore, $\delta$-TOGL provides a more user-friendly parametric
selection strategy. The purpose of this part is to  discuss how the
training time and testing time of $\delta$-TOGL vary with $\delta$.
Since the testing time depends only on the sparsity of the final
estimator, we use the number of iterations to replace the testing
time in this simulation.  In this simulation, we only take the level
of noise as $\sigma=0.1$ and the other experimental setting is
 the same as that of Sec.5.4. The simulation  results are reported in
the following Fig.7.

\begin{figure}[H]\addtolength{\tabcolsep}{-3pt}
\centering
\subfigure[]{\label{Fig.sub.a}\includegraphics[height=5cm,width=6cm]{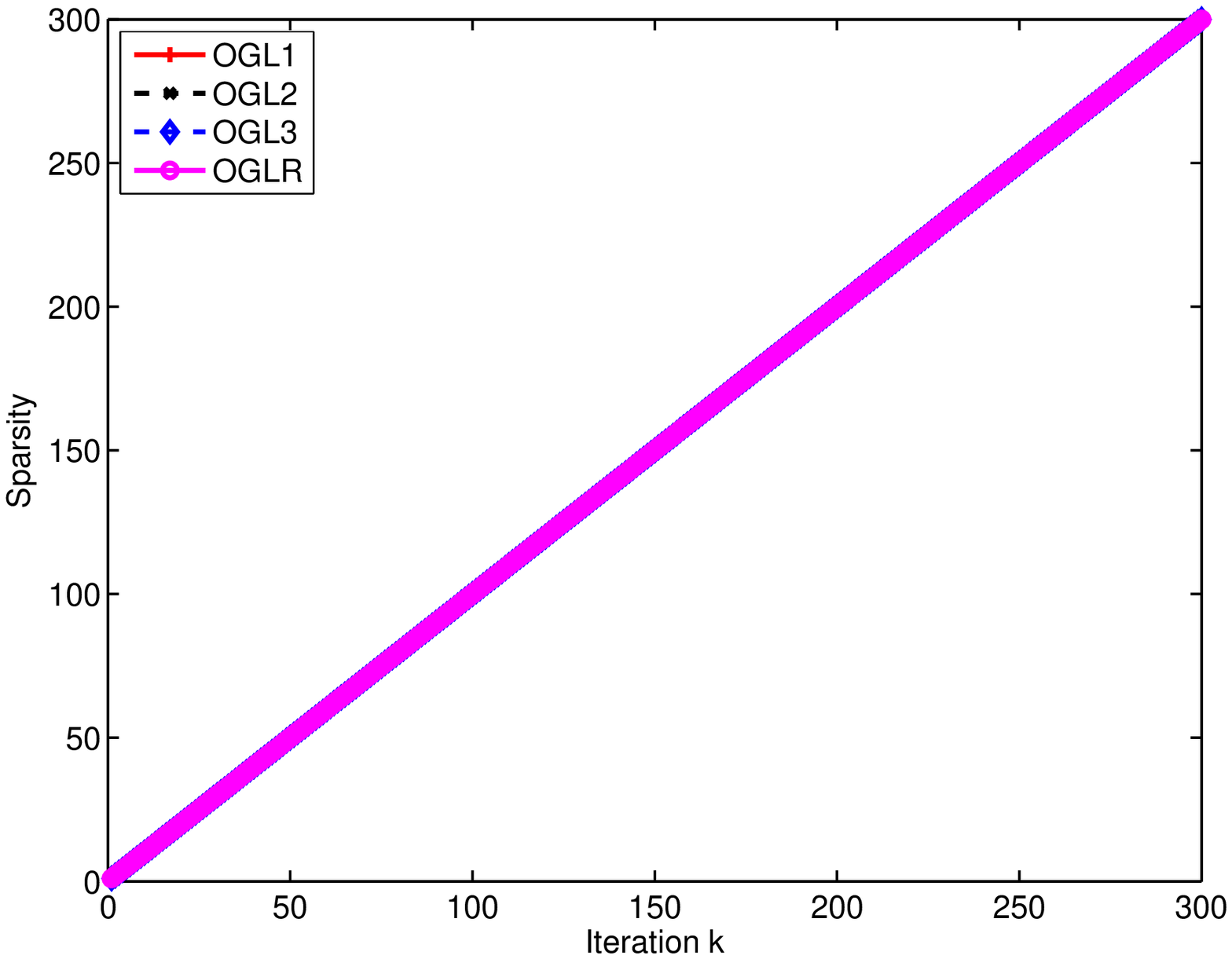}}
\subfigure[]{\label{Fig.sub.b}\includegraphics[height=5cm,width=6cm]{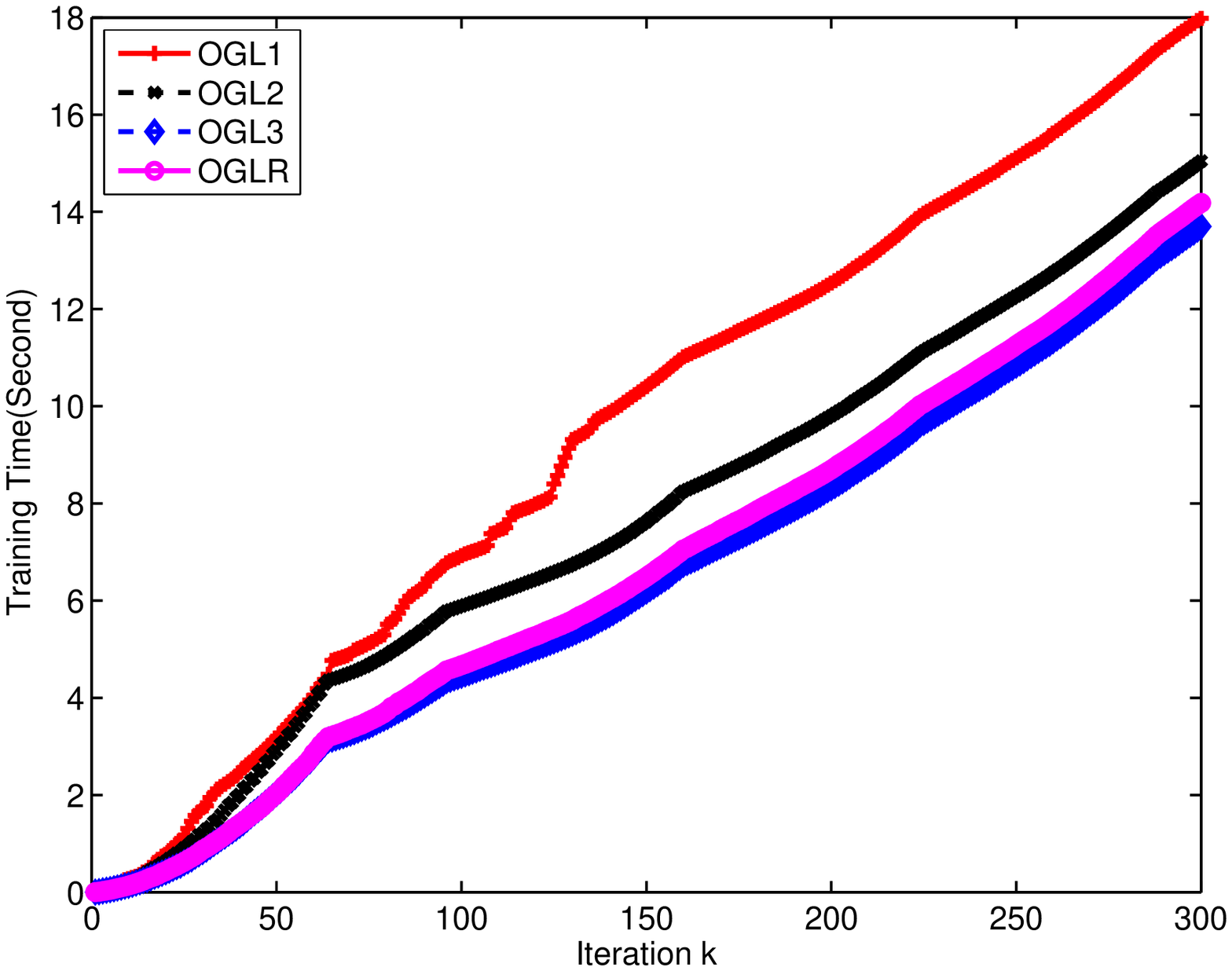}}
\subfigure[]{\label{Fig.sub.a}\includegraphics[height=5cm,width=6cm]{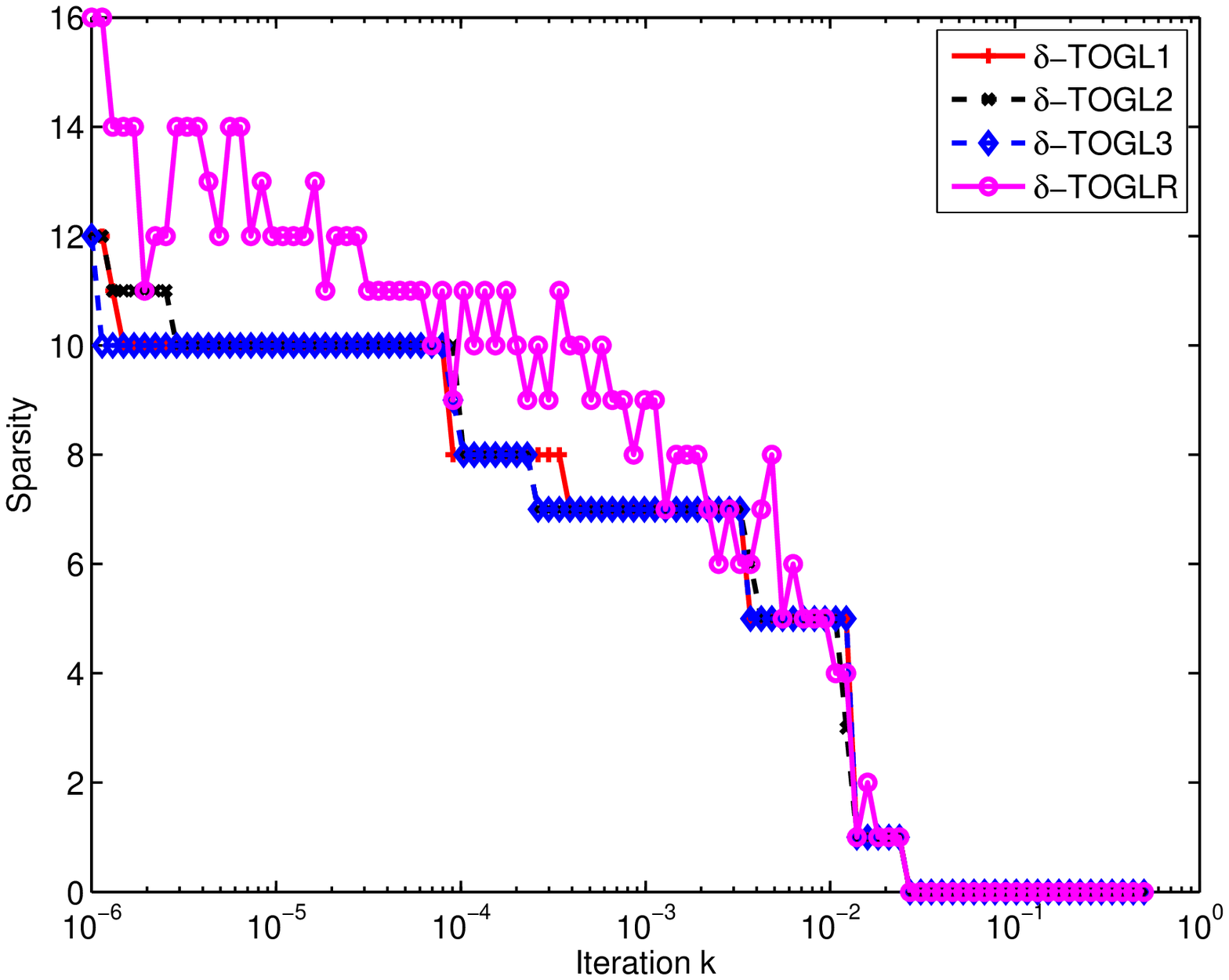}}
\subfigure[]{\label{Fig.sub.b}\includegraphics[height=5cm,width=6cm]{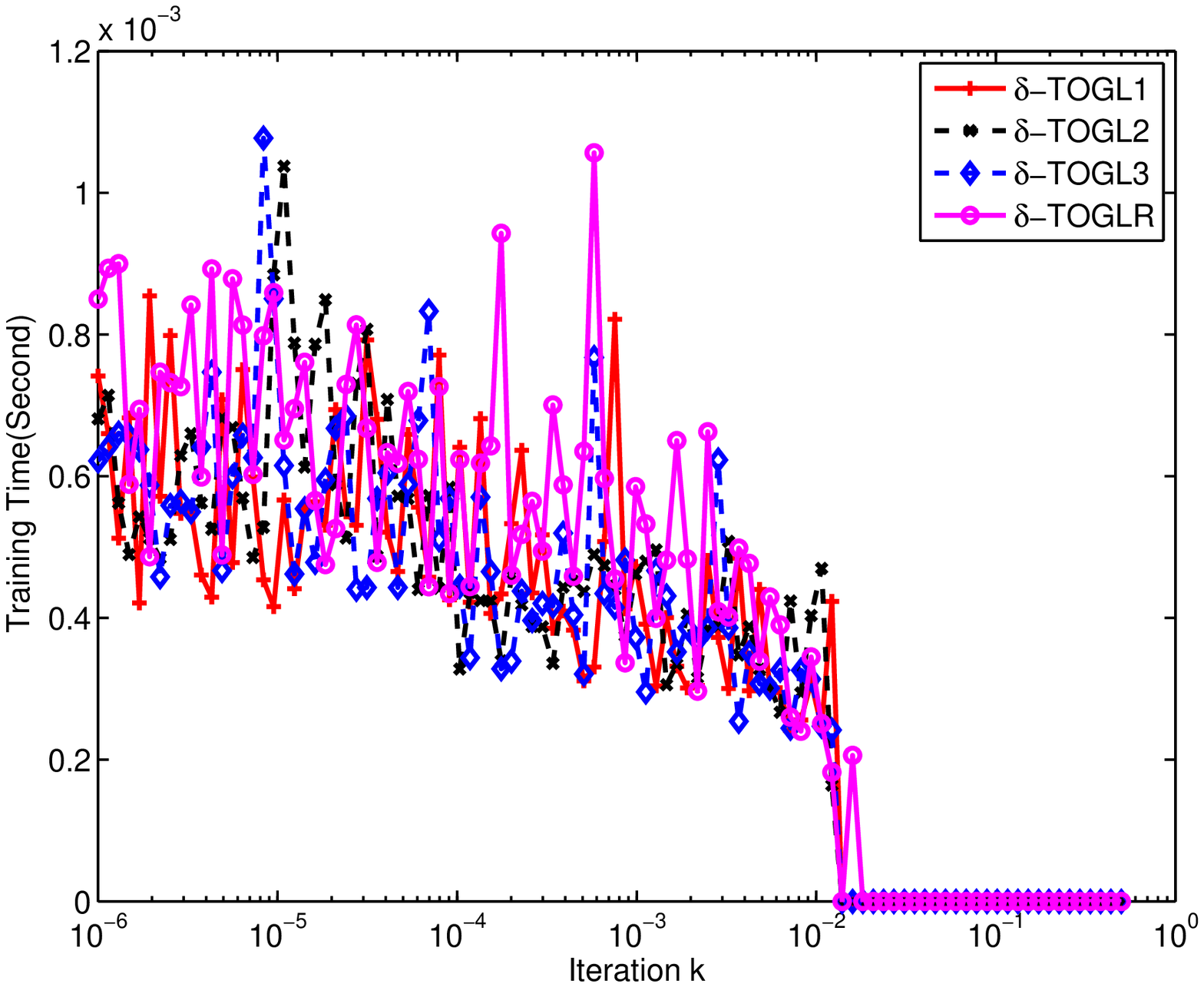}}
\caption{The parameter's influences on training and testing prices
in OGL and $\delta$-TOGL}
\end{figure}

From Fig.7, it shows that the  training and testing costs are  not
expensive  when the parameter $\delta$ tuning in the range
$[10^{-6},0.5]$, where the sparsity no more than 16 and the
corresponding training time is no more than $1.2 \times 10^{-3}$
second. All these show that when the parameter, $k$, of OGL is
transformed as $\delta$ in $\delta$-TOGL, both the training  and
test burdens are not added.

\subsection{ Usability and limitations of $\delta$-TOGL}

In this simulation experiments, we use $\delta$-TOGL1 to learn the
$sinc$ function with sampling noise as $\mathcal N(0,0.1^2)$. The
horizontal axis represents the number of training samples, and the
vertical axis represents the associated target accuracies (which
will be defined as follows). Therefore, every point in the
coordinate system denotes a given learning task. If the test RMSE of
$\delta$-TOGL with $\delta$ selecting by 5-fold cross-validation is
  less than the   accuracy, we define that the
learning task is successful and labeled 1, otherwise, the tasks
fails and tag 0. We run 100 times of trials in each point. The color
from blue to red denotes the values from $0$ to 100. The result is
shown in the following Fig.8.
\begin{figure}[H]
  \centering
  \includegraphics[height=6cm,width=7cm]{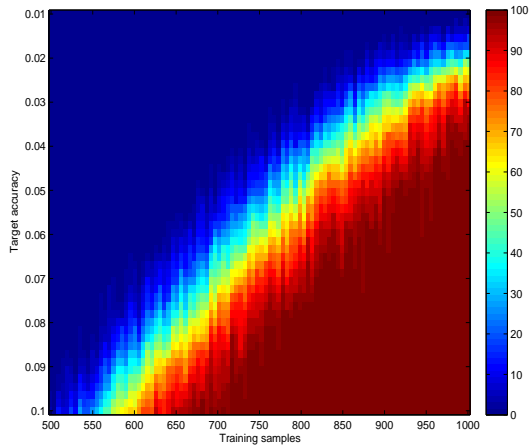}
\caption{Usability and limitations of $\delta$-TOGL}
\end{figure}

In the above Fig.8, the red areas represents that   $\delta$-TOGL
meets the demand of learning task  and  the blue area indicates
failure.   And we can immediately acquire an intuitive enlightenment
from the above phase transition diagram: given a set of data and a
target accuracy for a specific learning task, if you want to use
$\delta$-TOGL  to have a try, then such phase transition diagram can
tell you,  how many samples are approximately needed    to ensure
the accomplishment of your mission within a certain probability.
From the above  experimental result,   the generalization error of
$\delta$-TOGL    performances steadily, gradually inversely
monotonous to   the sample size, which fits   our theoretical
results in Theorem \ref{Theorem2}.

\begin{table}[H]
%\addtolength{\tabcolsep}{-10pt}
\begin{center}
 \caption{Compared $\delta$-TOGL performance with other classic algorithms.}\label{table6}
%\begin{minipage}{0.45\textwidth}
 \begin{tabular}{|c|c|c|c|}\hline
 $Methods$ & $Parameter $& $TestRMSE(standard error)$  &Sparsity   \\ \hline
  \multicolumn{4}{|c|}{$n=300$} \\ \hline
  OGL  &$k=9$&0.0218(0.0034) &9   \\ \hline
 $\delta$-TOGL1   &\text{$\delta=1.00e-4$}&0.0200(0.0044) &7.42   \\ \hline
 $\delta$-TOGL2   &\text{$\delta=2.00e-4$}&0.0203(0.0064) &8   \\ \hline
 $\delta$-TOGL3   &\text{$\delta=1.30e-6$}&0.0284(0.0074) &12.2   \\ \hline
 $\delta$-TOGLR   &\text{$\delta=3.80e-4$}&0.0219(0.0059) &9   \\ \hline
 $\mathcal L_2$(RLS)     &$\lambda=\text{5e-5}$&0.0263(0.0098) &300    \\ \hline
 $\mathcal L_1$(FISTA)   &$\lambda = \text{5e-6}$&0.0298(0.0092) &290.4   \\ \hline
  \multicolumn{4}{|c|}{$n=1000$} \\ \hline
    OGL  &$k=9$&0.0255(0.0045) &9   \\ \hline
 $\delta$-TOGL1    &\text{$\delta=1.00e-4$}&0.0277(0.0072) &7.2  \\ \hline
 $\delta$-TOGL2   &\text{$\delta=6.00e-4$}&0.0294(0.0119) &7   \\ \hline
 $\delta$-TOGL3   &\text{$\delta=6.00e-6$}&0.0211(0.0036) &7.8   \\ \hline
 $\delta$-TOGLR   &\text{$\delta=1.00e-4$}&0.0284(0.0082) &10.4   \\ \hline
 $\mathcal L_2$(RLS)  &$\lambda=0.0037$& 0.0272(0.0103)  &1000   \\ \hline
$\mathcal L_1$(FISTA)    &$\lambda=\text{7e-6}$&0.0277(0.0094)
&931.8  \\ \hline
  \multicolumn{4}{|c|}{$n=2000$} \\ \hline
    OGL   &$k=9$&0.0250(0.0054) &9   \\ \hline
 $\delta$-TOGL1    &\text{$\delta=2.00e-4$}&0.0256(0.0078)  &7.14  \\ \hline
 $\delta$-TOGL2   &\text{$\delta=1.00e-4$}&0.0280(0.0089) &8.6   \\ \hline
 $\delta$-TOGL3   &\text{$\delta=2.00e-6$}&0.0222(0.0082) &7.6   \\ \hline
 $\delta$-TOGLR   &\text{$\delta=9.06e-5$}&0.0266(0.0079) &10.6   \\ \hline
 $\mathcal L_2$(RLS)   &$\lambda= 0.0005$&0.0256(0.0126) &2000  \\ \hline
$\mathcal L_1$(FISTA)    &$\lambda=\text{7e-6}$&0.0235(0.0079) &1772
\\ \hline
\end{tabular}
%\end{minipage}
 \end{center}
 \end{table}
\subsection{ $\delta$-TOGL is competitive}

In this part, we compare $\delta$-TOGL with some classical
dictionary-based learning schemes such as the classical OGL,  ridge
and lasso estimators.   The regularization parameters of both ridge
and lasso estimators, the iteration number of OGL and the threshold,
$\delta$, of $\delta$-TOGL are drawn by using 5-fold
cross-validation. The regression is the $sinc$ function  with
sampling noise as the standard Gaussian noise with the variance
$0.1$, i.e., $\mathcal N(0,0.1^2)$.
 The simulation result can be seen in   Tab.4.

From Tab.4, we can see that  under the same order of generalization
performance  magnitude, the number of selected atoms of greedy-type
strategy is far smaller than the regularization    algorithms. This
explains why greedy-type algorithms are more suitable for redundant
dictionary learning \cite{Barron2008}. Furthermore, it also can be
found in Tab.4 that the generalization capability of all the
aforementioned learning schemes are similar. At last, our simulation
results shows that the size of dictionary doesn't affect the
learning performance of $\delta$-TOGL schemes very much, provided it
attains the lowest requirement to finishes the learning task.
 All these reveals that
$\delta$-OGL is a competitive learning scheme.

 \section{Proofs}

Since Theorem \ref{THEOREM1} can be regarded as a special case of
Theorem \ref{THEOREM2}, we only prove Theorem \ref{THEOREM2} in this
section. The methodology of proof is the same as that of
\cite{Lin2013a} and the main tool is borrowed from
\cite{Temlaykov2008a}.

 In order to give an error decomposition strategy for
$\mathcal E(f_{\bf z}^k)-\mathcal E(f_\rho)$, we  need to construct
a function $f_k^*\in \mbox{span}(D_n)$ as follows. Since $f_\rho\in
\mathcal L_1^r$, there exists a $h_\rho:=\sum_{i=1}^na_ig_i\in
\mbox{Span}(\mathcal D_n)$ such that
\begin{equation}\label{h}
                \|h_\rho\|_{\mathcal L_1}\leq\mathcal B,\ \mbox{and}\ \|f_\rho-h_\rho\|\leq \mathcal B n^{-r}.
\end{equation}
Define
\begin{equation}\label{f*}
               f_0^*=0,\  f_k^*=\left(1-\frac1k\right)f^*_{k-1}+\frac{\sum_{i=1}^n|a_i|\|g_i\|_\rho}{k}g^*_k,
\end{equation}
 where
$$
              g_k^*:=\arg\max\limits_{g\in \mathcal D_n'}\left\langle
            h_\rho-\left(1-\frac1k\right)f_{k-1}^*,g\right\rangle_{\rho},
$$
and
$$
            \mathcal D_n':=\left\{{g_i(x)}/{\|g_i\|_\rho}\right\}_{i=1}^n
             \bigcup
                 \left\{-{g_i(x)}/{\|g_i\|_\rho}\right\}_{i=1}^n
$$
with $g_i\in \mathcal D_n$.

Let $f_{\bf z}^\delta$  and $f_k^*$ be defined as in Algorithm 1 and
(\ref{f*}), respectively,  then we have
\begin{eqnarray*}
         &&\mathcal E(\pi_Mf_{\bf z}^\delta)-\mathcal E(f_\rho)\\
         &\leq&
         \mathcal E(f_k^*)-\mathcal E(f_\rho)
         +
         \mathcal E_{\bf z}(\pi_Mf_{\bf z}^\delta)-\mathcal E_{\bf z}(f_k^*)\\
         &+&
         \mathcal
        E_{\bf z}(f_k^*)-\mathcal E(f_k^*)+\mathcal E(\pi_Mf_{\bf z}^\delta)-\mathcal
        E_{\bf z}(f_{\bf z}^k),
\end{eqnarray*}
where $\mathcal E_{\bf
               z}(f)=\frac1m\sum_{i=1}^m(y_i-f(x_i))^2$.

Upon making the short hand notations
$$
             \mathcal D(k):=\mathcal E(f_k^*)-\mathcal E(f_\rho),
$$
 $$
          \mathcal S({\bf z},k,\delta):=\mathcal
           E_{\bf z}(f_k^*)-\mathcal E(f_k^*)+\mathcal E(\pi_Mf_{\bf z}^\delta)-\mathcal
           E_{\bf z}(\pi_Mf_{\bf z}^\delta),
$$
and
$$
            \mathcal P({\bf z},k,\delta):=\mathcal E_{\bf z}(\pi_Mf_{\bf z}^\delta)-\mathcal E_{\bf
            z}(f_k^*)
$$
respectively for the approximation error, the sample error and the
hypothesis error, we have
\begin{equation}\label{error decomposition}
            \mathcal E(\pi_Mf_{\bf z}^\delta)-\mathcal E(f_\rho)=\mathcal
            D(k)+ \mathcal S({\bf z},k,\delta)+\mathcal P({\bf z},k,\delta).
\end{equation}

At first, we give an upper bound estimate for $\mathcal D(k)$, which
can be found in Proposition 1 of \cite{Lin2013a}.

\begin{lemma}\label{LEMMA1}
 Let $f_k^*$ be defined in (\ref{f*}). If
$f_\rho\in \mathcal L_1^r$, then
\begin{equation}\label{approximation error estimation}
             \mathcal D(k)\leq  \mathcal B^2(k^{-1/2}+n^{-r})^2.
\end{equation}
\end{lemma}

To bound the sample and hypothesis errors, we need the following
Lemma \ref{LEMMA2}.

\begin{lemma}\label{LEMMA2}
Let $y(x)$ satisfy $y(x_i)=y_i$,  and $f_{\bf z}^\delta$ be  defined
  in Algorithm 1. Then, there are at most
\begin{equation}\label{Estimate k}
    C\delta^{-2}\log\frac1\delta
\end{equation}
bases selected to build up the estimator $f_{\bf z}^\delta$.
Furthermore, for any $h \in \mbox{Span}\{D_n\}$, we have
\begin{equation}\label{estimate hypothesis error}
           \|y - f_{\bf z}^\delta\|_m^2\leq2\|y - h\|_m^2+
       2\delta^2\|h\|_{\mathcal L_1(\mathcal D_n)}.
\end{equation}
\end{lemma}

\begin{proof}
(\ref{Estimate k}) can be found in \cite[Theorem
4.1]{Temlaykov2008a}. Now we turn to prove (\ref{estimate hypothesis
error}). Our stopping criterion guarantees that either $\max_{g\in
\mathcal D_n}|\langle r_{k},g\rangle_m|\leq\delta\|r_k\|_{m}$ or
$\|r_k\|\leq\delta\|y\|_m.$ In the latter case the required bound
follows form
$$
                \|y\|_m\leq\|y-h\|_m+\|h\|_m\leq\delta(\|y-h\|_m+\|h\|_m)
                \leq\delta(\|f-h\|_m+\|h\|_{\mathcal L_1(\mathcal D_n)}).
$$
Thus, we assume  $\max_{g\in \mathcal D_n}|\langle
r_{k},g\rangle_m|\leq\delta\|r_k\|_{m}$ holds. By using
$$
       \langle y-f_k,f_k\rangle_m=0,
$$
we have
\begin{eqnarray*}
     \|r_k\|_m^2
     &=&
     \langle r_k,r_k\rangle_m=\langle
     r_k,y-h\rangle_m+\langle r_k,h\rangle_m
     \leq
     \|y-h\|_m\|r_k\|_m+\langle r_k,h\rangle_m\\
     &\leq&
     \|y-h\|_m\|r_k\|_m+\|h\|_{\mathcal L_1(\mathcal D_n)}\max_{g\in \mathcal D_n}\langle
     r_k,g\rangle_m
     \leq
       \|y-h\|_m\|r_k\|_m+\|h\|_{\mathcal L_1(\mathcal D_n)}\delta\|r_k\|_m.
\end{eqnarray*}
This finishes the proof.
\end{proof}

Based on Lemma \ref{LEMMA2} and the fact $\|f^*_k\|_{\mathcal
L_1(\mathcal D_n)}\leq \mathcal B$ \cite[Lemma 1]{Lin2013a}, we
obtain
\begin{equation}\label{hypothesis error estimation}
            \mathcal P({\bf z},k,\delta)\leq 2\mathcal E_{\bf z}(\pi_Mf_{\bf z}^\delta)-\mathcal E_{\bf
            z}(f_k^*)\leq 2\mathcal B\delta^2.
\end{equation}

Now, we turn to    bound the sample error $\mathcal S({\bf z},k)$.
Upon using the short hand notations
$$
               S_1({\bf z},k):=\{\mathcal E_{\bf
               z}(f_k^*)-\mathcal E_{\bf
               z}(f_\rho)\}-\{\mathcal E(f_k^*)-\mathcal
               E(f_\rho)\}
$$
and
$$
               S_2({\bf z},\delta):=\{\mathcal E(\pi_Mf_{\bf z}^\delta)-\mathcal E(f_\rho)\}-\{\mathcal E_{\bf
               z}(\pi_Mf_{\bf z}^\delta)-\mathcal E_{\bf z}(f_\rho)\},
$$
we write
\begin{equation}\label{sample decomposition}
            \mathcal S({\bf z},k)=\mathcal S_1({\bf z},k)+\mathcal
            S_2({\bf z},\delta).
\end{equation}
It can be found in Proposition 2 of \cite{Lin2013a} that
 for any $0<t<1$, with confidence
$1-\frac{t}2$,
\begin{equation}\label{S1 estimate}
              \mathcal S_1({\bf z},k)\leq \frac{7(3M+\mathcal B\log\frac2t)}{3m}+\frac12\mathcal D(k)
\end{equation}

Using \cite[Eqs(A.10)]{Xu2014} with $k$ replaced by
$C\delta^{-2}\log\frac1\delta$, we have
\begin{equation}\label{S2 estimate}
      \mathcal S_2({\bf z},\delta)\leq \frac12\mathcal E(\pi_Mf_{\bf
      z}^\delta)-\mathcal
      E(f_\rho)+\log\frac2t\frac{C\delta^{-2}\log\frac1\delta\log
      m}{m}
\end{equation}
holds with confidence at least $1-t/2$. Therefore, (\ref{error
decomposition}), (\ref{approximation error estimation}),
(\ref{hypothesis error estimation}), (\ref{S1 estimate}), (\ref{S2
estimate}) and (\ref{sample decomposition}) yields that
$$
 \mathcal E(\pi_Mf_{\bf z}^\delta)-\mathcal E(f_\rho)
             \leq
             C\mathcal B^2( (m\delta^2)^{-1}\log m\log \frac{1}{\delta }\log\frac2t+\delta^2+n^{-2r})
$$
holds with confidence at least $1-t$. This finishes the proof of
Theorem \ref{THEOREM2}.

\section{Concluding Remarks }

The main contributions of the present paper can be concluded into
four folds. Firstly, we propose that the steepest gradient descent
(SGD) is not the unique choice  to select a new atom from dictionary
in orthogonal greedy algorithm (OGL), which    disrupts habitual
thinking to make a way for searching new greedy metric for OGL. To
the best of our knowledge, this is the first work on the
``greedy-metric'' issue for greedy learning. Secondly, we succeed in
finding an appropriate greedy metric in OGL and theoretically and
numerically verify its rationality and feasibility. Motivated by a
series work  of Temlyakov and his co-authors \cite{Liu2012},
\cite{Temlaykov2000,Temlaykov2003,Temlaykov2008,Temlaykov2008a},
 we propose a $\delta$-greedy thresholds to
measure the level of greed in orthogonal greedy learning. Our
theoretical result shows that orthogonal greedy learning with such a
greedy metric  yields a learning rate as $ m^{-1/2} (\log m)^2$,
which is almost the same as that of the classical SGD-based OGL
\cite{Barron2008}. Thirdly, based on the selected greedy metric, we
derive an adaptive terminal rule for the corresponding OGL and thus
provide a complete learning system called $\delta$-thresholding
orthogonal greedy learning ($\delta$-TOGL). Lastly, we study the
learning performance of $\delta$-TOGL in terms of both theoretical
analysis and numerical verification. Our study implies that
$\delta$-TOGL is a competitive learning scheme as the widely used
strategies such as the classical orthogonal greedy learning, ridge
estimate and lasso estimate. The main results show that when applied
to supervised learning problems, $\delta$-TOGL outperforms
dictionary-based regularization learning schemes such as lasso and
ridge regression in the sense that it can produces extremely high
sparseness of the final estimator. It also outperforms the classical
orthogonal greedy learning in the sense that it provides a more
user-friendly parametric selection strategy.

To stimulate more opinions from others on the ``greedy-metric''
issue of greedy learning, we present the following two remarks.

\begin{remark}\label{remark1}
 In this paper, we give a type of ``greedy-metric'' for OGL. In
 greedy approximation, Temlyakov \cite{Temlaykov2008} has been
 proposed various greedy-metric such as the super greedy algorithm
 and weak greedy algorithm. Since greedy learning focus on not only the
 approximation capability but also the capacity of the space spanned by the selected atoms,
  we guess that all these metrics can be
 adopted in greedy learning and may possess similar performances as
 the classical steepest gradient descent metric. We will also keep
working on this    issue and report our progress in a future
publication.
\end{remark}

\begin{remark}\label{remark1}
Programmers frequently ask us what is the essential advantage of
$\delta$-TOGL. This is a good question and we find a bit headache to
answer it. Admittedly, in this paper, we do not provide any
essential advantages of $\delta$-TOGL. The purpose of this paper is
only to propose the concepts of ``greedy metric'' and show that we
can use the greedy metric to reach the ``bias'' and ``variance''
trade-off. However, in our opinion,  there are at least two
advantages of $\delta$-TOGL. The first one is that, compared with
OGL, its generalization capability is not so sensitive to the
parameter. This advantage has already  been shown in Fig.4 and
Fig.5. The second one, $\delta$-TOGL can be viewed as an accelerated
version of OGL. As shown in Step 2 in Algorithm 1, we can select the
first atom satisfies the greedy metric. Under this circumstance, it
need not to compute   the $\langle r_{k-1},g\rangle_m$ for all
$g\in\mathcal D_n$. Once the size of dictionary is large, such an
operation can save a large number of computations. As the main
purpose of this paper is not to emphasize the computational speed,
we do not illustrate this advantage in the present paper. If it is
necessary, we will study this advantage within  practical
applications and report our progress in a future publication.
\end{remark}

\end{document}